%% file: main.tex
\documentclass{article}

\usepackage{microtype}
\usepackage{graphicx}
\usepackage{subfigure}
\usepackage{booktabs} 

\usepackage{hyperref}

\usepackage[accepted]{icml2024}

\usepackage{amsmath}
\usepackage{amssymb}
\usepackage{mathtools}
\usepackage{amsthm}

\usepackage[capitalize,noabbrev]{cleveref}

\theoremstyle{plain}
\newtheorem{theorem}{Theorem}[section]

\newtheorem{lemma}[theorem]{Lemma}

\theoremstyle{definition}
\newtheorem{definition}[theorem]{Definition}

\theoremstyle{remark}
\newtheorem{remark}[theorem]{Remark}

\usepackage{todonotes}

\usepackage{amsfonts}
\usepackage{graphicx}
\usepackage{booktabs}
\usepackage{multirow}
\usepackage{array}
\usepackage[labelfont=bf]{caption}
\usepackage{thmtools} 
\usepackage{thm-restate}
\input{math}

\newcolumntype{C}[1]{>{\centering\arraybackslash}m{#1}}
\usepackage{pifont}
\newcommand{\cmark}{\checkmark}%
\newcommand{\xmark}{\ding{53}}%
\DeclareRobustCommand{\parhead}[1]{\noindent \textbf{#1.} }

\icmltitlerunning{Stable Differentiable Causal Discovery}

\begin{document}

\twocolumn[
\icmltitle{Stable Differentiable Causal Discovery}



\icmlsetsymbol{equal}{*}

\begin{icmlauthorlist}
\icmlauthor{Achille Nazaret}{equal,cs,iicd}
\icmlauthor{Justin Hong}{equal,cs,iicd}
\icmlauthor{Elham Azizi}{cs,iicd,bme}
\icmlauthor{David Blei}{cs,stat}
\end{icmlauthorlist}

\icmlaffiliation{cs}{Department of Computer Science, Columbia University, New York, USA}
\icmlaffiliation{iicd}{Irving Institute for Cancer Dynamics, Columbia University, New York, USA}
\icmlaffiliation{bme}{Department of Biomedical Engineering, Columbia University, New York, USA}
\icmlaffiliation{stat}{Department of Statistics, Columbia University, New York, USA}

\icmlcorrespondingauthor{Elham Azizi}{ea2690@columbia.edu}

\icmlkeywords{Machine Learning, ICML}

\vskip 0.3in
]



\printAffiliationsAndNotice{\icmlEqualContribution} 

\begin{abstract}
  \input{tex/0-abstract}

\end{abstract}

\input{tex/1-introduction}

\input{tex/2-background}
\input{tex/3-acyclicity}

\input{tex/4-method}

\input{tex/5-empirical-studies}

\bibliography{main}
\bibliographystyle{icml2024}

\newpage
\appendix
\onecolumn
\input{supplement}

\end{document}

%% file: math.tex
\usepackage{stmaryrd}

\newcommand{\brackets}[1]{\left[ #1 \right]}
\newcommand{\parens}[1]{\left( #1 \right)}
\newcommand{\curly}[1]{\left\{ #1 \right\}}
\newcommand{\bbrackets}[1]{\llbracket #1 \rrbracket}
\newcommand{\tr}[1]{\textnormal{Tr}\brackets{#1}}

\newcommand*{\E}[2][]{\mathbb{E}\ifx\\\left[#1\right]\\\else_{#1}\fi \left[#2\right]}

\newcommand{\R}{\mathbb{R}}
\newcommand{\N}{\mathbb{N}}
\newcommand{\C}{\mathbb{C}}

\usepackage{bm}

\usepackage{dsfont}

\DeclareMathOperator*{\argmax}{arg\,max}

\newcommand{\pa}{\operatorname{pa}}
\newcommand{\bo}{\operatorname{bo}}

\newcommand{\Tr}{\textnormal{Tr}}

%% file: tex/0-abstract.tex
Inferring causal relationships as directed acyclic graphs (DAGs) is an important but challenging problem. 
Differentiable Causal Discovery (DCD) is a promising approach to this problem, framing the search as a continuous optimization. 
But existing DCD methods are numerically unstable, with poor performance beyond tens of variables. 
In this paper, we propose Stable Differentiable Causal Discovery (SDCD), a new method that improves previous DCD methods in two ways: 
(1) It employs an alternative constraint for acyclicity; this constraint is more stable, both theoretically and empirically, and fast to compute. 
(2) It uses a training procedure tailored for sparse causal graphs, which are common in real-world scenarios.
We first derive SDCD and prove its stability and correctness. 
We then evaluate it with both observational and interventional data and in both small-scale and large-scale settings. 
We find that SDCD outperforms existing methods in convergence speed and accuracy, and can scale to thousands of variables. We provide code at
\href{https://github.com/azizilab/sdcd}{github.com/azizilab/sdcd}.

%% file: tex/1-introduction.tex
\section{Introduction}

Inferring cause-and-effect relationships between variables is a fundamental challenge in many scientific fields, including biology \citep{sachs2005causal}, climate science \citep{zhang2011causality}, and economics \citep{hoover2006causality}. 
Mathematically, a set of causal relations can be represented with a directed acyclic 
graph (DAG) where nodes are variables, and directed edges indicate direct causal 
effects. The goal of causal discovery is to recover the graph from the observed 
data. The data can either be interventional, where some variables were purposely 
manipulated, or purely observational, where there has been no manipulation. 




The challenge of causal discovery is that searching for the true DAG underlying the data is an NP-hard problem. Exact methods are intractable, even for modest numbers of variables \citep{chickering1996learning}. Yet datasets in fields like biology routinely involve thousands of variables \citep{dixit2016perturb}.\looseness=-1

To address this problem, \citet{zheng2018dags} introduced differentiable causal discovery (DCD), which formulates the DAG search as a continuous optimization over the space of all graph adjacency matrices. An essential element of this strategy is an acyclicity constraint, in the form of a penalty, that guides an otherwise unconstrained search toward acyclic graphs.\looseness=-1

This optimization-oriented formulation often scales better than previous methods, and it has opened opportunities to harness neural networks \citep{lachapelle2019gradient}, incorporate interventional data \citep{brouillard2020differentiable}, and use matrix approximation techniques \citep{lopez2022large}. But, while promising, existing DCD methods still struggle to scale consistently beyond tens of variables, or they rely on approximations that limit their applicability (see \cref{sec:results}). 

In this paper, we study the problems of DCD and improve on it, so that it can scale more easily and apply to many types of causal discovery problems. We trace the issues with DCD to the instability of its objective function; in particular, properties of the acyclicity constraint it uses to find a DAG solution. We formalize this notion of stability, show that previous DCD methods are unstable, and then formulate a method that is stable and scalable.

%


%

In details, this paper makes several contributions. First, we present a unifying theoretical view of
existing acyclicity constraints, which explains their intrinsic numerical instability.
We then employ a constraint, the \textit{spectral acyclicity constraint} \citep{lee2019scaling}, that is both faster to compute and offers improved numerical stability. We prove its stability and corroborate its good properties with experiments. 


Finally, we develop \textit{Stable Differentiable Causal Discovery} (SDCD). SDCD is a two-stage optimization procedure for causal discovery that is stable and computationally efficient. In its first stage, it prunes edges without regard for acyclicity. In the second stage, it performs DCD with the spectral acyclicity constraint described above. We prove that the first stage of SDCD does not remove true edges, and we show empirically that it is faster and more accurate than a single-stage optimization.  SDCD removes key barriers that previously limited differentiable causal discovery to small problem sizes and application contexts.





In sum, the main contributions of this work are:
\begin{itemize}
    \item We develop a theoretical analysis of the acyclicity constraints used in DCD, and their numerical instabilities.
    \item We motivate an alternative acyclicity constraint with superior stability, both
     theoretically and empirically.
    \item We propose the SDCD method for efficient DCD. It
    leverages the stable constraint within a two-stage optimization procedure designed for training robustness. We prove that SDCD
    does not compromise accuracy. 
    \item We empirically study SDCD, and show that it efficiently solves problems involving thousands of variables. Compared to previous methods, SDCD achieves faster convergence and improved accuracy on observational and interventional data. 
    \item Code is available at \href{https://github.com/azizilab/sdcd}{github.com/azizilab/sdcd}.
\end{itemize}

\begin{figure}
    \centering
    \includegraphics[width=\linewidth]{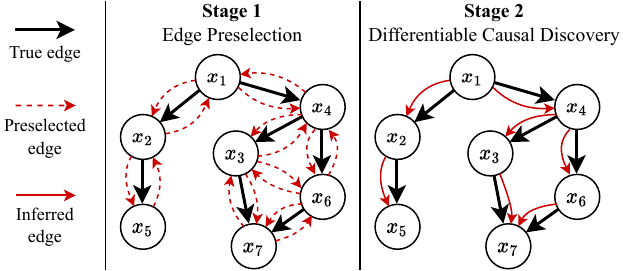}
   \caption{\protect\scalebox{0.98}{Visual representation of the SDCD method.}}
    \label{fig:cartoon}
\end{figure}

\parhead{Related Work}
\input{tex/1b-related-work}

\begin{table*}[ht]
\caption{Comparison of Differentiable Causal Discovery methods including our proposed SDCD method.
\textit{Expressive model class}
refers to the capability to approximate
any causal graph with non-linear structural equations.}
\label{table:method-comparison}
\begin{center}
\renewcommand{\arraystretch}{0}
\begin{tabular}{lC{2cm}C{2cm}C{2.4cm}C{2.3cm}}
\toprule
\textbf{Method} & \textbf{Stable Training} & \textbf{Scalable Constraint} & \textbf{Can Use Interventions} & \textbf{Expressive Model Class}    \\
\midrule
SDCD     & \cmark  & $\boldsymbol{O(d^2)}$ & \cmark   & \cmark  \\
\midrule
DCDI     & \xmark  & $O(d^3)$              & \cmark   & \cmark  \\
DCDFG    & \xmark  & $O(md)$               & \cmark   & \xmark  \\
DAGMA    & \xmark  & $O(d^3)$              & \xmark   & \cmark  \\
NO-TEARS & \xmark  & $O(d^3)$              & \xmark   & \xmark  \\
NO-BEARS & \cmark  & $\boldsymbol{O(d^2)}$ & \xmark   & \xmark  \\
\bottomrule
\end{tabular}
\end{center}
\end{table*}

%% file: tex/1b-related-work.tex
Causal discovery methods mainly fall into two categories: constraint-based methods and score-based methods \citep{glymour2019review}.

Constraint-based methods identify causal relationships by testing for conditional independence among variables in the data. For example, the PC algorithm \citep{spirtes2000causation} finds the graphs which conform to all the independencies present in the data. COmbINE is an extension to support interventional data \citep{triantafillou2015constraint}. 

On the other hand, score-based methods design a score $S(G)$ that is maximized by the true graph $G^*$, and they aim to find its maximizer $\hat G = \argmax_{G \in \textnormal{DAG}} S(G)$. Existing score-based methods differ in their choice of $S$ and of the optimization method to maximize it. GES \citep{chickering2002optimal} and GDS \citep{peters2014identifiability} optimize the BIC score of a Gaussian linear model by greedily adding or removing edges. GIES modifies GES to support interventional data \citep{hauser2012characterization}, and CAM supports non-linear additive models  \citep{buhlmann2014cam}. 

Differentiable Causal Discovery (DCD), which our work extends, is a type of score-based approach that reformulates the search of the score maximizer into a continuous optimization problem. It uses a numerical criterion to distinguish acyclic graphs from cyclic ones (called the acyclicity constraint). It was initially introduced in \citep{zheng2018dags} as NO-TEARS, which uses linear models, augmented Lagrangian optimization, and a constraint based on the adjacency matrix exponential. Other works extend the methodology to incorporate polynomial regression \citep{lee2019scaling}, neural networks \citep{lachapelle2019gradient,zheng2020learning}, and support for interventional data \citep{brouillard2020differentiable}. 

Alternative acyclicity constraints \citep{lee2019scaling,ng2020role,bello2022dagma} have been proposed, as well as optimization schemes different than augmented Lagrangian \citep{ng2020role,ng2022convergence,bello2022dagma}. \citet{deng2023optimizing} introduced a hybrid approach combining gradient optimization with combinatorial optimization, while \citet{lippe2021efficient} explored removing the acyclicity constraint entirely.\looseness=-1

Despite a rich literature, DCD has difficulty scaling to a large number of variables, exhibiting long training times and numerical instability. \citet{wei2020dags, ng2024structure} identified those limitations, and \citet{lee2019scaling,lopez2022large} addressed these problems with elegant approximations, but they resulted in poor accuracy.

Here, we provide a theoretical understanding of some algorithmic issues with DCD and then use that understanding to develop a better DCD method.
Our proposed method builds on ideas that have been partially studied in previous work, including the acyclicity constraint of \citet{lee2019scaling} and the penalty method of \citet{ng2020role}.
\citet{lee2019scaling} uses the spectral acyclicity constraint for computational efficiency but otherwise does not expand on its advantages. In this work, we provide a novel analysis of the constraint and incorporate it into a new strategy that is more accurate and scalable than existing DCD algorithms. \Cref{table:method-comparison} recapitulates research in DCD and how this work fits in.

%% file: tex/2-background.tex
\section{Background and Notations}
\label{sec:background}
We review the Differentiable Causal Discovery (DCD) approach and define the notations used in the paper.

\subsection{Background on Causal Discovery} 
Causal discovery is the task of learning cause-and-effect relationships among a set of variables.
In this work, we consider variables that can be intervened on such that they no longer are affected by their causal parents. 
These interventions are called \textit{structural} or \textit{perfect} interventions \citep{eberhardt2007interventions}.\looseness=-1

\parhead{Causal Graphical Models} Causal graphical models (CGMs) provide a mathematical framework for reasoning about causal relationships between variables. 
Consider a CGM over $d$ variables and $K$ possible interventions on it. 

There are three components:
\begin{enumerate}
    \item A directed acyclic graph (DAG),  $G^* = (V,E)$, where each node, $j \in V$, represents a variable $x_j$, and each edge, $(j, k) \in E$, indicates a direct causal relationship from $x_j$ to $x_k$.
    \item A list of conditional distributions, $p^*_{j}(x_j \mid x_{\pa^{G^*}_j}; 0)$, which specify the distribution of each $x_j$ given its causal parents $x_{\pa^{G^*}_j}$ without intervention (the 0 indicates no intervention).
    \item A list of interventions, $\mathcal{I} = \{ I_0, I_1, ..., I_K\}$, where $I_0 = \varnothing$ (no intervention) and the others $I_k \subset V$ define the target variables of intervention $k$. 
    Alongside is a list of interventional distributions, $p^*_{m}(x_m; k)$, for each $k>0$ and $m \in I_k$, which define the distributions over $x_m$ after intervention $k$.
\end{enumerate}
The joint distribution under intervention $k$ writes:
\begin{equation}
p^*(x; k) = \prod\limits_{j \in V\backslash I_k}  p^*_{j}(x_{j} \mid x_{\pa^{G^*}_j}; 0) \prod\limits_{j \in I_k} p_{j}^*(x_j; k).\label{eqn:interventional-distribution}
\end{equation}
Note $p^*\!(x; 0)$ is the joint on observational data.

\parhead{Data} We observe $n$ data points of the $d$ variables $X = \{(x_1^{i}, ..., x^{i}_d)\}_{i=1}^n$ with labels $T = \{t_i\}_{i=1}^n$ where $t^i \in \{0,...,K\}$ indicates which intervention was applied to $x^{i}$ ($0$ indicating no intervention). For example, in genomics, Perturb-seq screens \citep{replogle2022mapping} measure the expression of $d$ genes across $n$ cells, where each cell can be edited once to change the expression of one of its genes.


\parhead{Causal discovery with score-based methods} 
The goal of causal discovery is to infer the graph $G^*$ from the data $(X, T)$. 
In particular, score-based methods assign a score $S(G)$ to every possible graph $G$, where the score function is designed so that it is maximized on the true graph $G^*$. 
Score-based methods aim to find the maximizer
\begin{equation}
    \hat G = \argmax_{G \in \mathrm{ DAG }} S(G).
\end{equation}
Fix a model class that defines the conditional (and interventional) distributions for each possible DAG $G$ as $\{ p(\cdot \mid G; \theta, k)\}_\theta$. 
We can define the score $S(G)$ to be the maximum log-likelihood that can be achieved under graph $G$ with some regularization over the number of edges $|G|$ \citep{chickering2002optimal}:
\begin{align}
S_{\mathrm{mle}}(G) = \sup_{\theta} \left[\frac{1}{n} \sum_{i=1}^n \log p(x^i \mid G; \theta, t^i)\right] - \lambda |G|.
\label{eq:score-mle}
\end{align}

In the limit of infinite samples ($n \rightarrow \infty$), and under a few assumptions, any maximizer $\hat G$ of \Cref{eq:score-mle} is close to the true $G^*$ \citep{brouillard2020differentiable}. 
More precisely, $\hat G$ and $G^*$ are $\mathcal{I}$\textit{-Markov-equivalent}: they share the same skeleton, $v$-structures, and other restrictive properties at the intervened variables in $\mathcal{I}$ \citep{yang2018characterizing}.



\subsection{Differentiable Causal Discovery}
The main challenge to a score-based method for causal discovery is how to search over the large space of DAGs. Differentiable Causal Discovery (DCD) reformulates this combinatorial search into a continuous optimization problem over the space of all graphs, including cyclic ones \citep{zheng2018dags}. It introduces three key components.

\parhead{Model Class with Implicit Graph}
First, DCD defines a model class with no apparent underlying graph, where each variable is conditioned on all others as $p(\cdot \mid \theta, k) \propto \prod p_j(x_j | x_{-j}; \theta, k)$. Instead, $\theta$ defines the graph $G$ implicitly, such that if $\theta$ renders $x_{-j} \mapsto p_j(x_j | x_{-j}; \theta, 0)$ invariant to some $x_\ell \subset x_{-j}$, then there is no edge from $l$ to $j$. The induced adjacency matrix is denoted $A_\theta$. When $\theta$ induces an acyclic $A_{\theta}$, then $p(\cdot \mid \theta, k)$ defines a valid CGM.



\parhead{Acyclicity Function}
Second, DCD introduces a differentiable function $h(A_\theta)$ that quantifies how ``cyclic'' $A_\theta$ is. $h(A_\theta)$ is high when $A_\theta$ contains cycles with large edge weights, it is low when $A_\theta$ contains cycles with small weights, and $h(A_\theta) = 0$ when $A_\theta$ contains no cycles.

\parhead{Optimization} Finally, DCD reformulates \Cref{eq:score-mle} into a constrained optimization problem only over $\theta$. 
\begin{equation}
    \hat\theta = \argmax_{\substack{\theta \\ \textstyle \mathrm{~s.t.~} h(A_\theta) = 0}.} S_{\alpha, \beta}(\theta).
    \label{eq:dcd-argmax}
\end{equation}
It uses $A_\theta$ in place of $G$, uses the constraint $h(A_\theta) = 0$ in place of $G \in \mathrm{DAG}$, and uses a new objective $S_{\alpha, \beta}$:
\begin{equation}
    S_{\alpha,\beta}(\theta) = \frac{1}{n} \sum\limits_{i=1}^n \log p(x^i ; \theta, t^i) - \alpha \|A_\theta\|_1 - \beta \|\theta\|_2^2,
    \label{eq:dcd-score}
\end{equation}

$S_{\alpha, \beta}$ is a relaxed version of $S_{\mathrm{mle}}$ (\Cref{eq:score-mle}) where the discrete $|G|$ is changed into an L1 regularization of $A_\theta$ (for $\alpha \geq 0$) and an L2 regularization of $\theta$ is included (for $\beta \geq 0$). The $\sup_\theta$ in $S_{\mathrm{mle}}$ (\Cref{eq:score-mle}) is now removed, as it merges with the $\argmax_\theta$ of \Cref{eq:dcd-argmax}.

With \Cref{eq:dcd-argmax} in hand, different methods for DCD solve the constrained optimization in different ways. Some approaches use the augmented Lagrangian method \citep{zheng2018dags,lachapelle2019gradient,brouillard2020differentiable,lopez2022large}, some use the barrier method \citep{bello2022dagma}, and others use $h$ as a regularizing penalty \citep{ng2020role}. 

In all these approaches, the choices of $h$ and the optimization method dictate the optimization behavior and the ultimate quality of the inferred graph. In the next section, we highlight the importance of $h$.

%% file: tex/3-acyclicity.tex
\section{Stable Acyclicity Constraint}
\label{sec:acyclicity}
In this section, we demonstrate how most existing acyclicity constraints can lead to unstable numerical behaviors during optimization, especially with large numbers of variables $d$. We then motivate an alternative constraint, which we show to be theoretically and empirically more stable.

\subsection{Power Series Trace Constraints}
We first introduce a family of constraints. It generalizes existing constraints and reveals their similarities. 
\begin{definition}[The Power Series Trace Family]
    For any non-negative coefficients $(a_k)_{k\in \N^*} \in \R_{\geq 0}^{\N^*}$, consider the power series $f_a(x) = \sum\limits_{k=1}^\infty a_k x^k.$\\
    
    Then, for any matrix $A \in \R_{\geq 0}^{d\times d}$ with non-negative entries, we define the Power Series Trace (PST) function $$h_a(A) = \tr{f_a(A)} = \sum_{k=1}^\infty a_k \tr{A^k}.$$     
\end{definition}
The quantity $h_a(A)$ is closely related to the cycles in the graph represented by $A$. In $h_a(A)$, each $\tr{A^k}$ equals the total weight of all length-$k$ cycles in $A$ -- where the weight of a cycle is the product of its edge weights \citep{bapat2010graphs}. 
The next theorem generalizes the result of \citet{wei2020dags} to show that most $h_a$ can be used to characterize acyclicity.

\begin{restatable}[PST constraint]{theorem}{thPstConstraint}
\label{th:pst-constraint}
For any sequence $(a_k)_{k\in \N^*} \in \R_{\geq 0}^{\N^*}$, if we have $a_k > 0$ for all $k \in \bbrackets{1,d}$, then, for any matrix $A \in \R_{\geq0}^{d\times d}$, we have \looseness=-1
$$\left\{ \begin{array}{l}
    h_a(A) = 0 \Leftrightarrow A \text{ is acyclic},\\
    h_a(A) \geq 0,  \\
    \nabla h_{a}(A) = h_{a'}(A^\top) \text{ with } a'_k = (k+1)a_{k+1}.
\end{array}\right. $$
We say that $h_a$ is a PST constraint.
\end{restatable}
The proof is in \Cref{app:proof:pst-constraint}.
In particular, sequences of strictly positive $a_k$ satisfy the conditions for any $d$, so several standard power series are PST constraints. 

For example, the sequence $a^{\exp}_k = \frac{1}{k!}$ recovers the penalty $h_{\exp}(A) = \Tr(\exp(A)) - d$ originally proposed in \citet{zheng2018dags}.

If we define $a^{\log}_k = \frac{1}{k}$, then $f_{a^{\log}} = \sum_{k=1}^\infty \frac{x^k}{k}$ is the power series of $x \mapsto  - \log(1-x)$. With the identity $\Tr \log A = \log \det A$  \citep{withers2010log}, where $\log A$ is the matrix logarithm, we find that $$h_{\log}(A) = \Tr( - \log (I-A)) = - \log \det (I - A).$$ 
This is precisely the constraint introduced in \citet{ng2020role,bello2022dagma}. Hence, even though it uses the matrix determinant instead of the matrix trace, we uncover that $h_{\log}$ is also a PST constraint.


\begin{table}[t]
    \centering
    \resizebox{\columnwidth}{!}{%
    \begin{tabular}{ccccc}
    \toprule
    Name                 & $a_k$           & $f_a$                 & $h_a$              & $\nabla h_a^\top$  \\
    \midrule 
    $h_{\exp}$           & $1/k!$          & $\exp(x)-1$           & $\Tr\exp(A)-d$     & $\exp(A)$          \\
    $h_{\log}$           & $1/k$          & $\log\frac{1}{1-x}$   & $-\log\det(I-A)$   & $(I - A)^{-1}$     \\
    $h_{\mathrm{inv}}$   & $1$            & $\frac{1}{1-x}$       & $\Tr(I - A)^{-1}$  & $(I - A)^{-2}$     \\
    $h_{\mathrm{binom}}$ & $\binom{d}{k}$ & $(1+x)^d-1$           & $\Tr(I + A)^{d}-d$ & $d (I+A)^{d-1}$    \\
    \midrule
    $h_{\rho}$           & --          & --   & $|\lambda_d(A)|$   & $v_d  u_d^\top / v_d^\top  u_d$     \\
    
    \bottomrule
    \end{tabular}
    }
    \caption{(Top) Existing PST constraints with their power series and gradients. (Bottom) The spectral acyclicity constraint, which is not PST.}
    \label{tab:pst-constraints}
\end{table}

\Cref{tab:pst-constraints} shows that other constraints such as $h_{\text{binom}} = \Tr((I+A)^d) -d$ \citep{yu2019dag}, $h_{\text{inv}} = \Tr((I-A)^{-1}) -d$ \citep{zheng2018dags} are also PST.

\subsection{Limitations of PST constraints}
\label{sub:pst-limitations}
In this section, we provide the criteria necessary for constraints to exhibit stable optimization behavior. We prove that PST constraints do not satisfy these criteria and show empirically that optimization with these constraints can be slow or fail. As a solution, we suggest an alternative, non-PST acyclicity constraint and demonstrate its stability.

\begin{definition}
    An acyclicity constraint $h$ is stable if these three criteria hold for almost every $A \in \R_{\geq 0}^{d\times d}$:

    \parhead{\hspace*{0.1cm} - E-stable} $h(sA) = O_{s \rightarrow \infty}(s)$ \\
    \parhead{\hspace*{0.1cm} - V-stable} $h(A) \neq 0 \Rightarrow h(\varepsilon A) = \Omega_{\varepsilon \rightarrow 0^+}(\varepsilon)$ \\
    \parhead{\hspace*{0.1cm} - D-stable} $h$ and $\nabla h$ are defined almost everywhere. 
    \label{def:stability}
\end{definition}

\textit{E-stability} ensures that $h$ does not \textit{explode} to infinity; \textit{V-stability} ensures $h$ does not \textit{vanish} rapidly to 0; \textit{D-stability} ensures that $h$ and its gradient are well \textit{defined}.

These three criteria are all important for maximizing $S_{\alpha, \beta}(\theta)$ 
under the constraint $h(A_\theta) = 0$. 
D-stability and E-stability ensure the constraint remains well-defined and with bounded values throughout the optimization procedure. The V-stability is related to the nature of constrained optimization. Methods like augmented Lagrangian, barrier functions, and penalties use the constraint $h$ to formulate an objective of the form 
$S_{\alpha,\beta}(\theta) - \gamma h(A_\theta) - \mu h(A_\theta)^2$. They then increase $\gamma$ and $\mu$ until $h(A_\theta)$ reaches 0. These increments ensure that the penalty does not become negligible relative to $S_{\alpha,\beta}(\theta)$. 
But without V-stability, $h(A_\theta)$ can shrink quickly very close to 0, while $A_\theta$ remains far from a DAG. So, for full convergence, these methods must grow $\gamma$ and $\mu$ to large values, which can be either inefficient (it requires more training epochs) or fail (as $\gamma$ or $h(A_\theta)$ eventually reach the limit of machine precision). The studies in \Cref{sec:empirical.observational} demonstrate these failure modes happen in practice.


We now show that PST constraints are unstable, especially as $d$ grows.

\begin{restatable}[PST instability]{theorem}{thPstUnstable}
\label{th:pst-unstable}
    For $d \geq 2$, any PST constraint $h$ is both E-unstable and V-unstable. More precisely,\vspace{0.1cm}
 \parhead{\hspace*{0.1cm} - E-unstable} $\exists A \in \R_{\geq 0}^{d \times d}, h(s A) = \Omega_{s \rightarrow \infty}(s^d)$ \\
\parhead{\hspace*{0.1cm} - V-unstable} $\exists A \in \R_{\geq 0}^{d \times d}, h(\varepsilon A) = O_{\varepsilon \rightarrow 0^+}(\varepsilon^d)$ 
    
Also, any PST constraint for which $f_a$ has a finite radius of convergence is \textbf{D-unstable} (e.g., $h_{\log}, h_{\mathrm{inv}}$).
\end{restatable}

\Cref{th:pst-unstable} is proved in \Cref{app:proof:pst-unstable}. It shows that the instability of the PST constraints worsens exponentially in $d$.  \Cref{fig:constraint-asymptotics} empirically corroborates the theorem with two types of adjacency matrices encountered during DCD: a cycle and some uniformly random noise. It shows that all PST constraints escalate to infinity or vanish to zero as the scale of noise $\varepsilon$ changes (\Cref{fig:constraint-asymptotics} left)
or as the number of variables $d$ increases (\Cref{fig:constraint-asymptotics} right), reflecting their E-instability and V-instability. In addition, the D-instability of $h_{\log}$ and $h_{\text{inv}}$ appears even in small $\epsilon$ or $d$ (vertical lines). We encounter all three instabilities during causal discovery experiments (\Cref{sec:empirical.observational}), leading existing approaches to fail.\looseness=-1

\begin{figure}
    \centering
    \includegraphics[width=\linewidth]{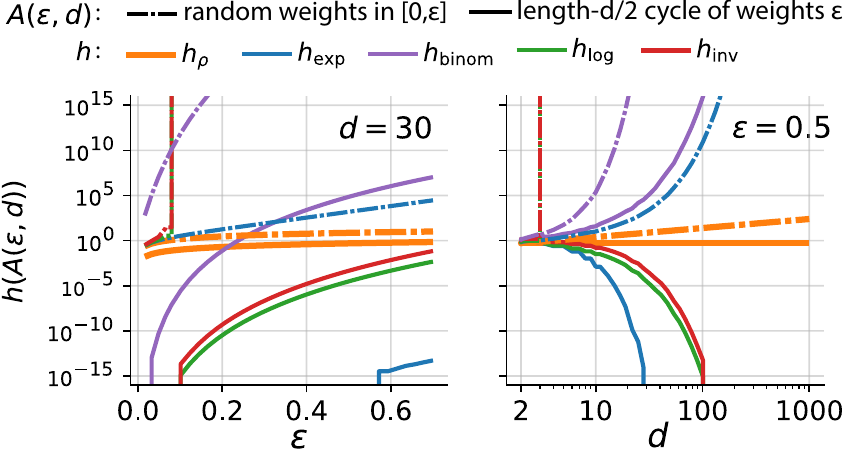}
    \caption{Constraint behaviors when evaluated on uniform random matrices in $[0, \epsilon]^{d \times d}$ (dashed) or a cycle of length $d/2$ with weight $\varepsilon$ (solid). The y-axis shows the constraint's value, the x-axis is (left) the weights' scale $\varepsilon$ (right) the number of variables $d$. Only the proposed $h_\rho$ (orange) remains stable; others vanish to zero exponentially or escalate to infinity (as soon as $d>10$). The vertical dotted lines indicate the constraint escaped its domain of definition. All these failures were encountered during DCD experiments.}
    \label{fig:constraint-asymptotics}
\end{figure}

\subsection{The Spectral Acyclicity Constraint}
\label{sub:spectral}

To overcome the limits of the PST constraint family, we propose to use another type of constraint, one based on the spectrum of $A$, which was first used in \citet{lee2019scaling}. This constraint draws from a characterization of DAG matrices from graph theory - that $A$ is acyclic if and only if all its eigenvalues are zero \citep{cvetkovic1980}.

We write $\lambda_1(A) \in \C$ to $\lambda_d(A) \in \C$, the $d$ eigenvalues of $A$, sorted from smallest to highest complex magnitude
\begin{definition}[Spectral radius]
The spectral radius
$$ h_{\rho}(A) = |\lambda_d(A)|,$$ is the largest eigenvalue magnitude of $A$.
\end{definition}

The next theorem shows that the spectral radius can be used as an acyclicity constraint.

\begin{restatable}[\citet{cvetkovic1980,lee2019scaling}]{theorem}{spectralAcyclicity}
\label{th:constraint-spectral-valid}
The spectral radius is an acyclicity constraint.
    $$h_\rho(A) = 0 \Leftrightarrow A \textnormal{ is a DAG. }$$ 
    We refer to it as the spectral acyclicity constraint. It is differentiable almost everywhere, with gradient $$\nabla h_{\rho} ( A ) = v_d  u_d^\top / v_d^\top  u_d,$$ where $u_d, v_d$ are respectively the right and left eigenvectors associated with $\lambda_d(A)$ \citep{magnus1985differentiating}. 
\end{restatable}

\Cref{th:constraint-spectral-valid} is proved in \Cref{app:proof:spectral-constraint}. It implies that $h_\rho$ is D-stable. Next, we prove $h_\rho$ is E-and-V-stable.

\begin{restatable}{theorem}{spectralStable}
\label{th:constraint-spectral-stable}
   $h_{\rho}$ is stable.
\end{restatable}

We refer to \Cref{app:proof:spectral-stable} for the proof. 

\begin{remark}
As a corollary of \Cref{th:constraint-spectral-stable}, $h_{\rho}$ is not another PST constraint (since it is stable). 
\end{remark}

We complete \Cref{fig:constraint-asymptotics} with the empirical behavior of $h_{\rho}$. As theoretically expected, $h_{\rho}$ retains non-extreme values and is suitable for constraint-based optimization.

To further understand the impact of the constraints' stability 
on optimization, we empirically study the optimization path of the augmented Lagrangian and the penalty method with each constraint in Appendix~\Cref{app:fig:constraint-optimization-behaviors}. The instabilities of PST constraints effectively slow their convergence and require increasing $\gamma$ and $\mu$ to excessively large values. In contrast, the optimization paths with $h_\rho$ take the least number of iterations to converge, especially with the penalty method. Moreover, the computation of $h_{\rho}$ can be done in $O(d^2)$ time (See Appendix~\ref{app:sdcd-algo}), contrary to the PST constraints whose computations scale in $O(d^3)$.

We are ready to perform DCD with the stable $h_\rho$.

%% file: tex/4-method.tex
\section{Stable Differentiable Causal Discovery}
\label{sec:method}

With the stable acyclicity constraint $h_\rho$ in hand, we now introduce Stable Differentiable Causal Discovery (SDCD). SDCD efficiently learns causal graphs in two stages.

\subsection{The SDCD method}

To solve the optimization problem (\ref{eq:dcd-argmax}) with the spectral acyclicity constraint $h_\rho$, SDCD optimizes the following objective with gradient-based optimization:
\begin{equation}
\textstyle
\hat\theta = \argmax\limits_{\theta} S_{\alpha, \beta}(\theta) - \gamma \cdot h_{\rho}(A_\theta),
\label{eq:score-sdcd}
\end{equation}
where $h_\rho$ is used as a penalty with coefficient $\gamma$. SDCD proceeds in two stages (See \Cref{fig:cartoon}).


\parhead{Stage 1: Edge Preselection} 
First, SDCD solves \Cref{eq:score-sdcd} without the constraint, by setting $\gamma = 0$.
\begin{equation}
    \hat\theta_1 = \argmax_{\substack{\theta \\ \forall j, A_{\theta, jj} = 0}} S_{\alpha_1, \beta_1}(\theta)
    \label{eq:stage-1}
\end{equation}
This stage amounts to solving simultaneously $d$ independent prediction problems of each variable given the others (the constraint $A_{\theta, jj} = 0$ prevents self-loops). The goal is to identify nonpredictive edges and remove them in stage 2, akin to \textit{feature selection}.

SDCD selects the \textit{removed edges} as  $\hat R_1 = \{ (j,l) \in \bbrackets{1,d}^2 \mid A_{\hat \theta_1, jl} < \tau_1\}$ where $\tau_1$ is a threshold.

\parhead{Stage 2: Differentiable Causal Discovery}
Next, SDCD re-solves \Cref{eq:score-sdcd}, this time with the constraint and with masking the removed edges from stage 1,
\begin{equation}
    \hat\theta_2 = \argmax_{\substack{\theta \\ \forall (j,l) \in \hat R_1, A_{\theta, jl} = 0}} S_{\alpha_2, \beta_2}(\theta) - \gamma_2 h_{\rho}(A_\theta).
    \label{eq:stage-2}
\end{equation}
The term $\gamma_2$ is initialized at 0 and is increased by a constant, $\gamma_\delta$, after each epoch.
Like other DCD methods, SDCD forms the final graph $\hat G_{\mathrm{SDCD}}$ by selecting the edges in $A_{\hat\theta_2}$ with weight above a threshold $\tau_2$. The details of the algorithm can be found in \Cref{app:sdcd-algo}. 

\begin{remark}
    In both stages, the constraints $A_{\theta, jl} = 0$ are straightforward to enforce by masking the elements in $\theta$ corresponding to $A_{\theta, jl}$ (i.e., fixing them at 0).
\end{remark}

Compared to other methods, SDCD innovates
in two ways: (1) by using the constraint $h_{\rho}$ with the penalty method and (2) by using a two-stage optimization that preselects edges in stage 1 and optimize the DCD objective only on those in stage 2. Without explicit masking, stage 2 would be similar to the barrier or penalty method \citep{ng2020role,bello2022dagma}, with stage 1 only providing a warm start.

\parhead{Motivations for stage 1}Dedicating stage 1 to removing unlikely edges is motivated by the hypothesis that real-life causal graphs are sparse. For example, individual genes in biological systems are typically regulated by a few other genes rather than all other genes \citep{lambert2018human}. A similar hypothesis underlies work in sparse mechanism shift \citep{scholkopf2021toward}.
Hence, stage 1 will likely remove many false edges and facilitate stage 2.
Alternative approaches for variable selection (e.g., markov boundaries \citep{loh2014high, wang2012learning}, skeletons \citep{tsamardinos2006max}, preliminary neighborhood selection \citep{buhlmann2014cam, lachapelle2019gradient}) can be motivated for the same reasons. Here, we found a simple modification to the objective function can effectively serve this purpose.
In practice, we find that stage 1 improves convergence speed and accuracy (\Cref{sec:results}, \Cref{app:tab:empirical:ablation}).
Notably, we find that stage 1 improves the stability of the training in stage 2, even when PST constraints are used in place of the spectral one (\Cref{app:tab:empirical:ablation2}). For this reason, stage 1 may also serve as a beneficial preprocessing step for other causal discovery methods.
In \Cref{th:stage1} below, we prove that stage 1 does not remove true causal parents. 



\subsection{Theoretical guarantees}
We analyze SDCD's time complexity in \Cref{app:sec:time_space_complexity}.
We now provide correctness guarantees for the two stages of SDCD.
We show that theoretically, stage 1 does not remove true causal parents, and so, stage 2 returns an optimal graph.

As done in the field (e.g., \citet{chickering2002optimal,brouillard2020differentiable}), the results focus on the ``theoretical'' $\hat G$ that would be obtained with infinite data and if \Cref{eq:stage-1,eq:stage-2} were solved exactly, in their non-relaxed form. 
We study SDCD in practice in \Cref{sec:results}.

With infinite data, \Cref{eq:stage-1}'s unrelaxed version writes,
\begin{equation}
\resizebox{0.9\hsize}{!}{$\displaystyle%
    \tilde \theta_1 = \!\argmax_{\substack{\theta \\ \forall j, A_{\theta, jj} = 0}} \sum\limits_{\substack{k=0\\ I_k \not\ni j}}^K \!\pi_k \!\!\! \underset{p^{*}\!(x;k)}{\mathbb{E}}\!\!\left[\log p_{j}(x_j | x_{-j}; \theta,0)\right]\! -\! \lambda |A_\theta|,$}
\end{equation}
where $\pi_k$ is the proportion of data coming from intervention $k$. The next theorem characterizes the graph $\tilde G_1 = A_{\tilde \theta_1}$ in terms of Markov boundaries in the true graph $G^*$. A Markov boundary for $j$ is a minimal set of variables that render $j$ independent of all the others. In a causal graph, each $j$ has a unique Markov boundary, consisting of $j$'s parents, $j$'s children, and $j$'s children's parents \citep{neapolitan2004learning}. 

\begin{restatable}{theorem}{theoremMarkovBoundary}
\label{th:stage1}
Under regularity assumptions detailed in \Cref{app:proof:stage-one}, the candidate parents $\pa_j^{\tilde G_1}$ of $j$ selected by stage 1 are precisely the Markov boundary of $j$ in the true graph $G^*$, That is, $\pa_j^{\tilde G_1} = \pa_j^{G^*} \cup ~\textnormal{ch}_j^{G^*} \cup \pa_{\textnormal{ch}_j^{G^*}}^{G^*}$.
\end{restatable}

The assumptions of \Cref{th:stage1} and its proof are detailed in \Cref{app:proof:stage-one}. The assumptions are reasonable: $p^*$ should be in the model class $\{p_\theta\}$, the expectations should be well defined, and ``faithfulness'' should hold (that is, $G^*$ doesn't have superfluous edges).

\Cref{th:stage1} gives two guarantees: (1) stage 1 does not remove causal parents and (2) stage 1 returns only a subset of the edges, not all of them. For instance, if $G^*$ is sparse such that each node has at most $k$ parents, then only $O(dk^2)$ edges are returned, which is essentially linear in $d$ for small $k$ (see \Cref{app:proof:stage-one-lemma}).

\Cref{th:stage1} implies that \citet[Theorem 1]{brouillard2020differentiable} still applies, and we deduce that stage 2 remains optimal under the stated assumptions (see \Cref{app:proof:stage-two}).


The theoretical results are reassuring. In the next section, we study SDCD's empirical performance to examine the impact of
finite data, nonconvex optimization, and relaxations.

%% file: tex/5-empirical-studies.tex
\section{Empirical studies}
\label{sec:results}

We compare SDCD to state-of-the-art baselines on multiple datasets. We find that SDCD achieves significantly better scores in both observational and interventional settings, particularly excelling at recovering sparser graphs. SDCD is the only method to scale to thousands of variables without sacrificing accuracy. 

\subsection{Evaluation Setup}
\label{sec:empirical-sdcd}

\parhead{Baselines for interventional data}
For datasets with interventional data, we compare SDCD against DCDI \citep{brouillard2020differentiable}, DCD-FG \citep{lopez2022large}, and GIES \citep{hauser2012characterization}. 


\parhead{Baselines for observational data} When the dataset contains only observational data, we include the interventional methods and further compare against NO-TEARS \citep{zheng2018dags}, NO-BEARS \citep{lee2019scaling}, DAGMA \citep{bello2022dagma}, and SCORE \citep{rolland2022score}. In addition, we report \textit{sortnregress} \citep{reisach2021beware}, a trivial baseline that should be outperformed (see Robustness Checks). We further included NOCURL \citep{yu2021dags} and AVICI \citep{lorch2022amortized} in \Cref{app:fig:empirical:observational}.


\parhead{Metrics}
We evaluate performance using the structural Hamming distance (SHD) between the true $G^*$ and each method's output graph. SHD is standard in causal discovery. It quantifies the minimum number of edge additions, deletions, and reversals needed to transform one graph into the other. Lower SHDs indicate better reconstructions of $G^*$.

\parhead{Robustness Checks}
Previous works detailed common issues with the SHD metric \citep{tsamardinos2006max} and data simulation processes \citep{reisach2021beware}. We include additional metrics and baselines recommended by previous works to ensure our evaluation is robust. 
Further details and results are detailed in \Cref{app:empirical:robustness}.

\parhead{Data}
We simulate observational and interventional data for a wide range of $d$ (number of variables), varying the graph density with $s$ (the average number of parents per node), and varying the number of variables that are intervened on. The simulations proceed as done in \citet{brouillard2020differentiable,bello2022dagma}, by sampling a random graph, modeling its conditionals with random neural networks, setting its interventional distribution to Gaussian, and drawing samples from the obtained model.
More details are in \Cref{app:empirical:simulations-details}. In all experiments, the number of observational samples is fixed at $10,000$, and an additional $500$ samples are added for each perturbed variable.

To further validate the results against the strongest baseline, we evaluate SDCD on the simulated data generated in \citet{brouillard2020differentiable} (DCDI) and compare our results against their reported SHD values.

\parhead{Setting} Consistent with prior work (e.g., DAGMA, NOTEARS), we do not conduct hyperparameter optimization for the experiments. Instead, we fix a single set of parameters for all experiments (see \Cref{app:empirical:hyperparameters}). The training time on CPU is measured on an AMD 3960x with 4-core per method; on GPU on an AMD 3960x with 16-core and an Nvidia A5000.

\parhead{SDCD Modeling Assumptions}We use neural networks (NNs) to parameterize the model class, as done in \citet{lachapelle2019gradient,zheng2020learning}. Each $p_j(x_j \mid x_{-j} ; \theta, k)$ is a Gaussian distribution over $x_j$ with mean and variance given by an NN as a function of all the other $x_{-j}$. 
More details about the NN architecture are in \Cref{app:sdcd-model}.
Also, SDCD is amenable to other model classes, such as normalizing flows \citep{brouillard2020differentiable}.

\begin{figure}[t]
    \centering
    \includegraphics[width=\linewidth]{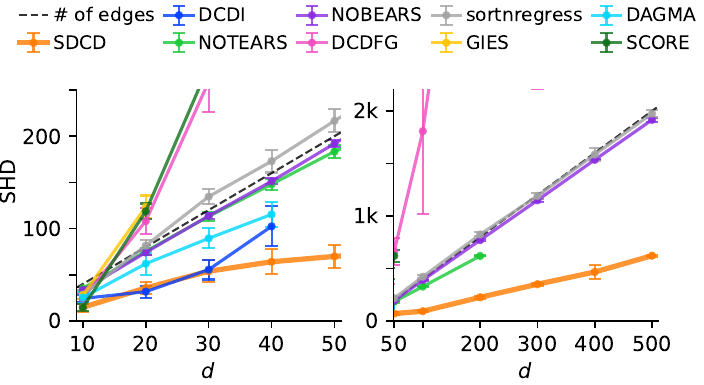}
    \caption{SHD across simulations on observational data with increasing numbers of variables $d$. SDCD achieves the best SHDs. It is the only method scaling above 200 variables with nontrivial SHD. Missing data points imply the method failed to run. Error bars indicate std on 30 random datasets for $d \leq 50$ and five for $d>50$ (175 total). Lower is better.}
    \label{fig:observational}
\end{figure}

\begin{figure}
    \centering
    \includegraphics[width=0.95\linewidth]{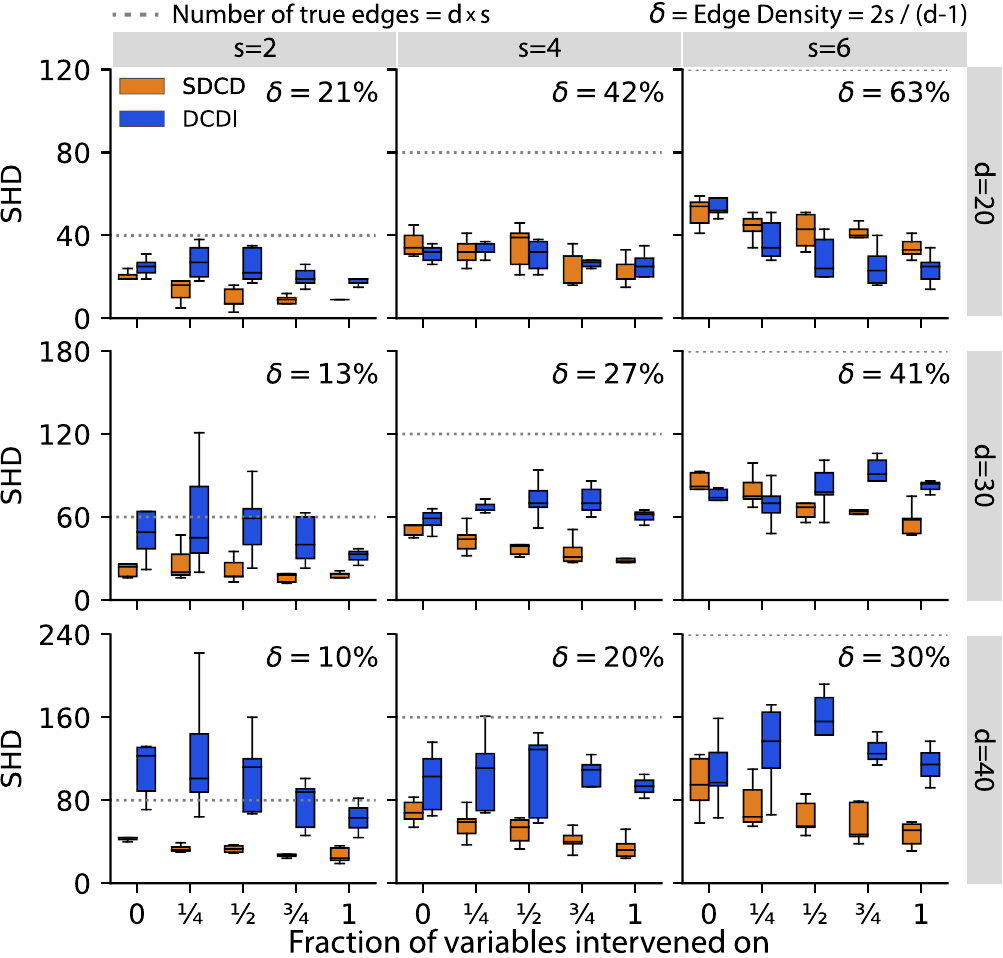}
    \caption{SHD across simulations with an increasing proportion of variables intervened on, varying the total number of variables $d$ (columns) and average edges per variable $s$ (rows). 
    SDCD is the only method to consistently improve with interventional data and has the best SHDs for sparse graphs (edge density $\delta \leq 45\%$).
    Each boxplot over 5 random datasets (45 datasets total).
    \looseness=-1}
    \label{fig:interventional}
\end{figure}

\subsection{Observational Data Experiments}
\label{sec:empirical.observational}

We evaluate all eight methods on a wide range of number of variables
$d$, with a fixed average number of edges per variable $s = 4$,
and repeat the experiments over 30 random datasets. \Cref{fig:observational} reports the results and detailed tables are provided in \Cref{app:fig:empirical:observational} with additional baselines.

SDCD outperforms the other methods in accuracy at every scale and speed. It can be explained by SDCD's stability.

\parhead{Failures of other methods} DCDI is competitive on small $d$ but crashes for $d>40$ -- as discussed in \ref{sub:pst-limitations}, for $d=50$, NaNs appear during training when $h_{\text{exp}}$ underflows due to V-instability; for $d>50$ NaNs appear right at initialization when $h_{\text{exp}}$ overflows due to E-instability. DAGMA fails to converge within 6 hours for as few as 30 variables due to the learned adjacency matrix escaping the domain of definition of $h_{\log}$, caused by D-instability. DAGMA attempts to stay within the domain of definition of $h_{\log}$ by reducing the learning rate near the singularities, but this is often not sufficient and it significantly slows down training.
NO-TEARS and NO-BEARS perform similarly to the trivial baseline \textit{sortnregress}, confirming the findings of \citet{reisach2021beware}. DCD-FG scales well but has exceptionally high SHD due to predicting very dense graphs -- which we attribute to its low-rank approximation. 

Finally, we note that most methods outperform the SHD of the empty graph (it is the number of edges, as dashed line).

To show that SDCD's performance is robust to a comprehensive set of scenarios, we provide additional metrics for these experiments in \Cref{app:empirical:robustness,app:fig:empirical:observational-f1-cpdag,app:fig:empirical:observational_recall}.

The runtimes associated with \Cref{fig:observational} are presented in Appendix~\Cref{app:fig:empirical:runtime}. SDCD-GPU runs under 15 minutes for all values of $d$ in \Cref{fig:observational} experiments (e.g., $d=500$). In Appendix~\Cref{app:fig:empirical:observational-large}, we further demonstrate that SDCD can scale up to 4,000 variables under 2h45.

\subsection{Interventional Data Experiments}

Next, we compare SDCD, DCDI, DCD-FG, and GIES over datasets with an increasing proportion of intervened variables. We show the results for SDCD and DCDI in Figure~\ref{fig:interventional} and all methods in the Appendix (DCD-FG and GIES performed consistently worse). As expected, the methods generally improve with more interventional data, although SDCD is the only method to do so consistently. We find that SDCD performs the best in most scenarios, particularly on sparser graphs. We characterize the edge density of a graph, $\delta$, as the ratio of true edges to the maximum number of edges possible in a DAG. 

\subsection{Ablation Experiments}
We performed ablation studies to judge the impact of each innovation implemented in SDCD. We evaluated modifications of SDCD where (1) only stage 2 is performed without stage 1 and where (2) in stage 2, the spectral constraint is substituted for alternative PST acyclicity constraints. As the results show in \Cref{app:tab:empirical:ablation,app:tab:empirical:ablation2}, both stages are essential to the success of SDCD.

\subsection{Experiment against the best baseline}
In Supplementary~\Cref{app:tab:empirical:dcdi} we report the results of SDCD on the simulated data presented in \cite{brouillard2020differentiable} alongside their original DCDI results. SDCD outperforms DCDI on all its sparse datasets ($s=1$). Only for datasets where $d=10, s=4$, does SDCD perform worse than DCDI. However, we find the edge density ($\delta = 88.9\%$) of these graphs to be unrepresentative of realistic scenarios.

\section{Conclusion}
With SDCD, we addressed the limitations of existing DCD methods by applying an acyclicity constraint and a two-stage procedure that each promotes stability.
We show it improves in all regimes and can scale to thousands of variables, enabling new applications for DCD in data-rich settings.

Future work could aim to provide a deeper theoretical understanding of the impact of the acyclicity constraint's stability on gradient-based optimization, particularly how the constraint's non-convexity affects training.

\section*{Acknowledgments}
A.N. was supported by funding from the Eric and Wendy Schmidt Center at the Broad Institute of MIT and Harvard, and the Africk Family Fund. J.H. was supported by grant number 2022-253560 from the Chan Zuckerberg Initiative DAF, an advised fund of the Silicon Valley Community Foundation, and the Irving Institute for Cancer Dynamics. E.A. was supported by the National Institute of Health (NIH) NCI grant R00CA230195 and NHGRI grant R01HG012875. D.B. was funded by NSF 2127869, NSF 2311108, ONR N00014-17-1-2131, ONR N00014-15-1-2209, the Simons Foundation, and Open Philanthropy.

\section*{Impact Statement}
This paper presents work whose goal is to advance the field of Machine Learning. There are many potential societal consequences of our work, none of which we feel must be specifically highlighted here.

%% file: supplement.tex
\icmltitle{Supplementary Materials: Stable Differentiable Causal Discovery}

\section{Theoretical Results}
\label{app:sec:proofs}
\subsection{Proof of \Cref{th:pst-constraint}}
\label{app:proof:pst-constraint}
Before proving \Cref{th:pst-constraint}, we precisely define an acyclic matrix and prove a few lemmas.

\begin{definition}[Cyclic and acyclic matrices]
\label{app:proof:pst-constrain:definition-cyclic}
    Take a matrix $A \in \R_{\geq 0}^{d \times d}$.
    
    We say that $A$ has a \emph{cycle of length $k$} if and only if:
    \begin{equation}
        \exists (i_0, ..., i_{k}) \in \bbrackets{1,k}^{k+1}, ~\textnormal{ such that, } \left\{ \begin{array}{l}
             i_0 = i_k  \\
             \forall \ell \in \bbrackets{1,k}, A_{i_{\ell - 1}, i_\ell} > 0
        \end{array}
        \right.
    \end{equation}

    We say that $A$ is \emph{cyclic} if it contains at least one cycle. We note that if $A$ contains a cycle of length $k$ for $k \in \N^*$ (the set of strictly positive integers), then $A$ also contains a cycle of length $k'$ for $k' \in \bbrackets{1,d}$ (this follows from the pigeon hole principle).

    We say that $A$ is \emph{acyclic} if it does not contain any cycle (or equivalently if it does not contain any cycle of length $k \leq d$).  
\end{definition}
\begin{lemma}
\label{app:proof:pst-constraint:lemma-cycles}
    For any matrix $A \in \R^{d \times d}_{\geq 0}$, 
    \begin{itemize}
        \item $\tr{A^k} \geq 0$ for any $k$
        \item $A$ has a cycle of length $k$ if and only if $\tr{A^k} > 0$
        \item $A^d = 0$ if and only if $A$ is acyclic.
    \end{itemize}
\end{lemma}

\begin{proof}
    Fix a matrix $A \in \R^{d \times d}_{\geq 0}$.
    We have, \begin{equation}
        \tr{A^k} = \sum_{\substack{(i_0, ..., i_k) \in \bbrackets{1,d}^{k+1} \\ i_0 = i_k = i}} \prod_{\ell=1}^k A_{\ell -1, \ell}.
    \end{equation}
    Each addend is non-negative so $\tr{A^k} \geq 0$. Furthermore, the total sum is strictly positive if and only if at least one addend is strictly positive. This happens if and only if $A$ has a cycle of length $k$ by definition.

    Similarly, we have \begin{align}
        (A^d)_{i,j} &= \sum_{\substack{(i_0, ..., i_d) \in \bbrackets{1,d}^{d+1} \\ i_0 = i_d = j}} \prod_{\ell=1}^k A_{\ell -1, \ell}.
    \end{align}
    If $A^d_{i,j} > 0$, then one addend is strictly positive and so there exists $(i_0, ..., i_d) \in \bbrackets{1,d}^{d+1}$ such that $i_0 = i_d = j$ and $\prod_{\ell=1}^k A_{\ell -1, \ell} >0$. By the pigeon-hole principle, two $i_\ell$ are identical, which provides a cycle. Reciprocally, if $(i_0, ..., i_k)$ is a cycle of length $k$, then by repeating $(i_0, ..., i_k, i_1, i_2 ...i_{d \mod k})$ until having a path of length $d+1$, we have that $(A^d)_{i_0, i_{d \mod k}} > 0$. Hence, $A^d = 0$ if and only if $A$ is acyclic.
\end{proof}

We recall \Cref{th:pst-constraint}.

\thPstConstraint*

\begin{proof}
    Fix a matrix $A \in \R^{d \times d}_{\geq 0}$ and a sequence $(a_k)_{k \in \N^*} \in \R^{\N^*}_{\geq 0}$ such that $a_k > 0$ for any $k \in \bbrackets{1, d}$.

    By definition, we have,
    \begin{align}
        h_a(A) &= \tr{\sum_{k=1}^{+\infty} a_k A^k}\\
        &= \sum_{k=1}^{+\infty}a_k  \tr{A^k}\\
    \end{align}
    \begin{enumerate}
        \item By \Cref{app:proof:pst-constraint:lemma-cycles}, $\tr{A^k} \geq 0$ and so $h_a(A) \geq 0$. This proves the second property.
        \item Then, $h_a(A) = 0$ if and only if $\tr{A^k} = 0$ for all $k$ for which $a_k > 0$. Since $a_k > 0$ for any $k \in \bbrackets{1, d}$ we conclude that if $h_a(A) = 0$, then $A$ does not contain cycles of length $k \leq d$, so $A$ is acyclic by \Cref{app:proof:pst-constrain:definition-cyclic}. Reciprocally, if $A$ is acyclic, it does not contain cycles of any length, so $h_a(A) = 0$.
        \item Finally, if we write $r_a$ the radius of convergence of $f_a$, then $h_a$ converge absolutely over the set of matrices with $h_\rho(A) < r_a$ so it is differentiable with gradient given by: $\nabla h_a(A) = \sum_{k=1}^{+\infty} a_k k (A^{\top})^{k-1} $.
    \end{enumerate}
    This concludes the proof.
    
\end{proof}

\subsection{Proof of Theorem 2}
\label{app:proof:pst-unstable}
We recall \Cref{th:pst-unstable}.

\thPstUnstable*
\begin{proof}
    Take a PST constraint $h_a$ for some $(a_k)_k \in \R^{d\times d}_{\geq 0}$ with $a_k > 0$ for $k \in \bbrackets{1,d}$.

    We will show the E-unstable and V-unstable results using a particular adjacency matrix $C$.

    Define $C$ as the adjacency matrix of the cycle $1 \rightarrow 2 \rightarrow ... \rightarrow d \rightarrow 1$ with edges weights of $1$. That is:
        \begin{equation}
            C = \begin{bmatrix}
                0 & 1 & 0 & \hdots &\hdots & 0 \\
                  & 0 & 1 & 0 & \hdots & 0 \\
                \vdots&&\ddots&\ddots&\ddots& \vdots \\
                &\vdots&&\ddots&\ddots& 0 \\
                0 &&&&\ddots& 1 \\
                1 & 0 &&\hdots && 0
            \end{bmatrix}
        \end{equation}
    We have $C^d = I_d$ and $\tr{C^k} = \left \{ \begin{array}{ll}
        d & \textnormal{ if } k = 0 \mod d \\
        0 &  \textnormal{ if } k \neq 0 \mod d
    \end{array}\right.$.
    
    We obtain for any $w \in \R_{\geq 0}$,
    \begin{equation}
        h_a(wC) = d \sum_{\ell = 1}^{+\infty} a_{\ell d} w^{\ell d}.
    \end{equation}
    \begin{itemize}
        \item In particular, we have for any $s \geq 0$, $h_a(wC) = d a_{d} s^{d} = \Omega_{s \rightarrow +\infty} s^d$ (since $a_d > 0$). This proves the E-instability.
        \item Define $u = \min(1, r_a/2)$ where $r_a$ is the radius of convergence of $f_a$. 
        
        Then, for any $\varepsilon \in [0, u^2]$,

        \begin{align}
            h_a(\varepsilon C) &= d \sum_{\ell = 1}^{+\infty} a_{\ell d} \varepsilon^{\ell d} \\
            &= \varepsilon^d d \parens{\sum_{\ell = 1}^{+\infty} a_{\ell d} \varepsilon^{(\ell - 1) d}} \\
            &\leq \varepsilon^d d \parens{\sum_{\ell = 1}^{+\infty} a_{\ell d} u^{2 (\ell - 1) d}} \label{eq:pst-unstable-1}\\
            &\leq \varepsilon^d d \parens{\sum_{\ell = 1}^{+\infty} a_{\ell d} u^{\ell d} + a_d} \label{eq:pst-unstable-2} \\
            &\leq \varepsilon^d d \parens{f_a(u C) + a_d} \label{eq:pst-unstable-3}. \\
            &= O_{\varepsilon \rightarrow 0^+} (\varepsilon^d)
        \end{align}
    Where we obtain \Cref{eq:pst-unstable-2} by noting that $2(\ell - 1) \geq \ell$ and $u \leq 1$. Finally, since $u < r_a$, $f_a(uC)$ is finite. Hence the result.
    \end{itemize}

    The D-instability result follows from the definition of the radius of convergence. 
\end{proof}

\subsection{Proof of Theorem 3}
\label{app:proof:spectral-constraint}
We recall \Cref{th:constraint-spectral-valid}. 
\spectralAcyclicity*

The two properties stated in \Cref{th:constraint-spectral-valid} are standard results.

\begin{proof}~\\
    \begin{itemize}
        \item We provide proof for the statement $h_\rho(A) = 0 \Leftrightarrow A \textnormal{ is a DAG }$ for the sake of completeness.
        \begin{itemize}
            \item[$\Rightarrow$] If $h_{\rho}(A) = 0$ then all eigenvalues $\lambda_j(A)$ are zeros. But since $\tr{A^k} = \sum_{j=1}^d \lambda_j(A)^k$, we have $\tr{A^k} = 0$ for any $k \geq 1$ and by \Cref{app:proof:pst-constraint:lemma-cycles}, $A$ is acyclic.
            \item[$\Leftarrow$] Assume $A$ is acyclic, then $A^d = 0$ by \Cref{app:proof:pst-constraint:lemma-cycles}. But then all eigenvalues are 0 (as for eigenvalue $\lambda_j(A)$ and associated eigenvector $v_j(A)$, we have $A^d v_j(A) = \lambda_j(A)^d v_j(A) = 0$.
        \end{itemize}
    Hence, $h_\rho$ is a valid acyclicity constraint.

        \item \citet{magnus1985differentiating} shows that $h_\rho$ is differentiable at every $A$ that has mutually distinct eigenvalues, with the formula provided in \Cref{th:constraint-spectral-stable}. The set of matrices with all distinct eigenvalues is dense in the set of matrices \citep{horn2012matrix}[Theorem 2.4.7.1], which proves the result 
    \end{itemize}
\end{proof}

\label{app:proof:spectral-constraint-corollary}
\subsection{Proof of Theorem 4}
\label{app:proof:spectral-stable}
We recall \Cref{th:constraint-spectral-stable}.
\spectralStable*

\begin{proof} We prove each stability criterion.

    \begin{itemize}
        \item \textbf{E-stable}: For any $s > 0$ and matrix $A$, $h_\rho(sA) = |s| h_\rho(A) = O_{s\rightarrow +\infty }(s)$.
        \item \textbf{V-stable}:For any $\varepsilon > 0$ and matrix $A$ such that $h_{\rho}(A) > 0$, $h_\rho(\varepsilon A) = |\varepsilon| h_\rho(A) = \Omega_{\varepsilon \rightarrow 0^+ }(\varepsilon)$.
        \item \textbf{D-stable}: Every matrix has eigenvalues ($\mathbb{C}$ is algebraically closed), so $h_\rho$ is well defined everywhere. In addition, \Cref{th:constraint-spectral-valid} proved that $h_\rho$ was differentiable almost everywhere. 
    \end{itemize}
    Hence, $h_\rho$ is a stable constraint.
\end{proof}

\subsection{Proof for Stage 1}
\label{app:proof:stage-one}

In this section, we guarantee that if the optimization problem solved in stage 1 is solved exactly, without relaxation and with infinite data, then stage 1 does not remove any true causal parent.

The optimization problem solved in stage 1 is given in \Cref{eq:stage-1} as
\begin{equation*}
    \hat\theta_1 = \argmax_{\substack{\theta \\ \forall j, A_{\theta, jj} = 0}} S_{\alpha_1, \beta_1}(\theta) = \argmax_{\substack{\theta \\ \forall j, A_{\theta, jj} = 0}} \frac{1}{n} \sum\limits_{i=1}^n \log p(x^i ; \theta, t^i) - \alpha_1 \|A_\theta\|_1 - \beta_2 \|\theta\|_2^2.
\end{equation*}

With infinite data $x^i | t^i \sim p^*(x^i; t^i)$, the optimization problem writes,
\begin{equation}
\displaystyle%
    \tilde \theta_1 = \argmax_{\substack{\theta \\ \forall j, A_{\theta, jj} = 0}} \sum\limits_{\substack{k=0}}^K \pi_k \underset{p^{*} (x;k)}{\mathbb{E}}\left[\log p(x; \theta, k)\right]  - \alpha_1 \|A_\theta\|_1 - \beta_2 \|\theta\|_2^2. 
    \label{app:eq:stage1:infinite}
\end{equation}
where $\pi_k$ is the proportion of data coming from intervention $k$. 

Furthermore, in its non-relaxed form, \Cref{app:eq:stage1:infinite} above writes
\begin{equation}\displaystyle%
    \tilde \theta_1 = \argmax_{\substack{\theta \\ \forall j, A_{\theta, jj} = 0}} \sum\limits_{\substack{k=0}}^K \pi_k \underset{p^{*} (x;k)}{\mathbb{E}}\left[\log p(x; \theta, k)\right]  - \lambda |A_\theta|, 
    \label{app:eq:stage1:infinite-nonrelaxed}
\end{equation}
where the L1 and L2 regularization are reverted back into the number of edges $|A_\theta|$ regularization (for some $\lambda > 0$).

Since we are interested in the graph induced by $\tilde \theta_1$, that we write $\tilde G_1 = A_{\tilde \theta_1}$, we can rewrite \Cref{app:eq:stage1:infinite-nonrelaxed} as
\begin{equation}\displaystyle%
    \tilde G_1 = \argmax_{\substack{G \\ \textnormal{without self-loops}} }\sup_{\substack{\theta\\ G = A_\theta}} \sum\limits_{\substack{k=0}}^K \pi_k \underset{p^{*}(x;k)}{\mathbb{E}}\left[\log p(x; \theta, k)\right]  - \lambda |G|, 
    \label{app:eq:stage1:infinite-nonrelaxed-graph}
\end{equation}

Finally, since there are no constraint over $G$ other than no self-loops, \Cref{app:eq:stage1:infinite-nonrelaxed-graph} can be solved as $d$ independent optimization problems, each one determining the parents of $j$ in the graph $\tilde G_1$,
\begin{equation}\displaystyle%
    \pa^{\tilde G_1}_j = \argmax_{\substack{S \subset \bbrackets{1,d} \backslash \{ j \}} }\sup_{\substack{\theta\\ S = \pa_j^{A_\theta}}} \sum\limits_{\substack{k=0}}^K \pi_k \underset{p^{*}(x;k)}{\mathbb{E}}\left[\log p_{j}(x_j | x_{-j}; \theta,k)\right]  - \lambda |S|, 
    \label{app:eq:stage1:infinite-nonrelaxed-parents}
\end{equation}

Furthermore, whenever $j \in I_k$, our model class has $p_j(x_j | x_{-j}, \theta, k) = p_j(x_j | \theta_{(j,k)}, k)$ --- we know we have perfect interventions and the interventions are known. So the $\theta_{(j,k)}$ is not related to the coordinates of $\theta$ that define $A_\theta$. That is to say, \cref{app:eq:stage1:infinite-nonrelaxed-parents} is equivalent to 
\begin{equation}\displaystyle%
    \pa^{\tilde G_1}_j = \argmax_{\substack{S \subset \bbrackets{1,d} \backslash \{ j \}} }\sup_{\substack{\theta\\ S = \pa_j^{A_\theta}}} \sum\limits_{\substack{k=0 \\ j \not \in I_k}}^K \pi_k \underset{p^{*}(x;k)}{\mathbb{E}}\left[\log p_{j}(x_j | x_{-j}; \theta,k)\right]  - \lambda |S|, 
    \label{app:eq:stage1:infinite-nonrelaxed-parents-bis}
\end{equation}

We recall \Cref{th:stage1}.

\theoremMarkovBoundary*

The assumptions are similar to the ones detailed in \citet{brouillard2020differentiable} to guarantee that differentiable causal discovery can identify causal graphs.

The assumptions are:
\begin{itemize}
    \item $\pi_0 > 0$ -- we observe some observational data,
    \item $\exists \theta, ~\textnormal{s.t.} ~ \forall k, p^*(\cdot~; k) = p(\cdot~; \theta, k)$ -- the model class can express the true model $p^*$,
    \item The observational distribution $p^*(x; 0)$ is \textit{faithful} to the graph $G^*$ (that is any edge in $G^*$ indeed result in a nonzero cause-and-effect relation in the distribution $p^*(x;0)$. See \citet{neapolitan2004learning} for more details.
    \item The true distributions $p^*(x; k)$ and any distribution of the model class $p(x; \theta, k)$ have strictly positive density $p^*(x; k) > 0$, $p(x; \theta, k) > 0$. This avoids technical difficulty when forming conditional distributions (e.g., $p^*(x_j | x_T; k)$).
    \item The expectations $\mathbb{E}_{p^*(x; k)}[ \log p^*(x; k) ]$ are well defined (they are finite). This enables us to consider the likelihood expectations in the first place.
    \item The regularization strength $\lambda$ is strictly positive and small enough (see the proof for how small).
\end{itemize}

\begin{proof}
Fix $j\in \bbrackets{1,j}$.

For clarity of notations, we rewrite \Cref{app:eq:stage1:infinite-nonrelaxed-parents-bis} as
\begin{equation}\displaystyle%
    \pa^{\tilde G_1}_j = \argmax_{\substack{S \subset \bbrackets{1,d} \backslash \{ j \}} }\sup_{\substack{\theta}} \sum\limits_{\substack{k=0 \\ j \not \in I_k}}^K \pi_k \underset{p^{*}(x;k)}{\mathbb{E}}\left[\log p_{j}(x_j | x_S; \theta,k)\right]  - \lambda |S|, 
    \label{app:eq:stage1:infinite-nonrelaxed-parents-ter}
\end{equation}
where the condition $S = \pa_j^{A_\theta}$ is fully captured by the notation $p_{j}(x_j | x_{S}; \theta,k)$.

Then, define 
\begin{equation}
    \psi(T) = \sup_{\theta} \sum_{\substack{k\\j \not\in I_k}} \E[p^*(x; k)]{\pi_k \log p_j(x_j | x_S; \theta, k)} - \lambda |S|.
\end{equation} 
Further, define $B = \bo^{G^*}_j$ to be the Markov boundary of node $j$ in the true causal graph $G^*$. 

We will show that $\psi(B) > \psi(T)$ for any other $T \subset \bbrackets{1,d}\backslash\curly{j}$.
    
    We compute,
    \begin{align}
        \psi(B) - \psi(T) &= \sup_{\theta} \sum_{\substack{k\\j \not\in I_k}} \pi_k \E[p^*(x; k)]{ \log p_j(x_j | x_B; \theta, k)}
        - \sup_{\theta} \sum_{\substack{k\\j \not\in I_k}} \pi_k \E[p^*(x; k)]{\log p_j(x_j ; x_T; \theta, k)} 
        \\&\quad- \lambda |B| + \lambda|T| \nonumber \\
        &= -\inf_{\theta} \sum_{\substack{k\\j \not\in I_k}} \pi_k \E[p^*(x_{-j} ; k)]{D_{KL}\parens{p_j^*(x_j | x_{-j}; k) ~\|~ p_j(x_j | x_B; \theta, k)}}  \label{eqn:big1}\\
        &\quad +\inf_{\theta} \sum_{\substack{k\\j \not\in I_k}} \pi_k  \E[p^*(x_{-j} ; k)]{D_{KL}\parens{p_j^*(x_j | x_{-j}; k) ~\|~ p_j(x_j ; x_T; \theta, k)}}  \nonumber
        \\&\quad- \lambda (|B| - |T|)\nonumber\\
        &= \inf_{\theta} \sum_{\substack{k\\j \not\in I_k}} \pi_k  \E[p^*(x_{-j} ; k)]{D_{KL}\parens{p_j^*(x_j | x_{-j}; k) ~\|~ p_j(x_j ; x_T; \theta, k)}}  \label{eqn:big2}
        \\&\quad+ \lambda (|T| - |B|).\nonumber
    \end{align}
    
    The line \ref{eqn:big1} comes from $\E[p^*(x; k)]{\log p_j(x_j | x_B; \theta, k)} = - \E[p^*(x_{-j}; k)]{D_{KL}\parens{p_j^*(x_j | x_{-j}; k) ~\|~ p_j(x_j | x_B; \theta, k)}} + \E[p^*(x; k)]{\log p_j^*(x_j | x_{-j}; k)}$ where we added and substracted the $\log p^{(k)}$ term (the $\mathbb{E}_{p^*(x; k)}$ is decomposed into $\mathbb{E}_{p^*(x_{-j}; k)}\mathbb{E}_{p^*(x_{j}; k)}$, where the second expectation is in the KL divergence). We use the assumption of strictly positive density here to define the conditional $p_j^*(x_j | x_{-j}; k)$ without technical difficulties. 
    
    The line \ref{eqn:big2} comes from the assumption of sufficient model class capacity and the definition of the Markov boundary. Indeed, we first have $p_j^*(x_j | x_{-j}; k) = p_j^*(x_j | x_{B}; k)$ by definition of the Markov boundary $B$, and since the model class is expressive enough, there exists $\theta$ such that $D_{KL}\parens{p_j^*(x_j | x_{-j}; k) ~\|~ p_j(x_j | x_B; \theta, k)} = 0$. 

    We further have:
    \begin{align}
        \psi(B) - \psi(T) &\geq \pi_0 \inf_{\theta} \E[p^*(x_{-j} ; 0)]{D_{KL}\parens{p_j^*(x_j | x_{B}; 0) ~\|~ p_j(x_j |  x_{T};  \theta, 0)}} + \lambda (|T| - |B|)
        \\
        &=  \pi_0 \E[p^*(x_{-j} ; 0)]{D_{KL}\parens{p^*_j(x_j | x_B; 0) ~\|~ p_j^*(x_j | x_T; 0)}} +  \pi_0\inf_{\theta} \E[p^*(x; 0)]{\log \frac{p^*_j(x_j | x_T; 0)}{p_j(x_j | x_{T};  \theta, 0)}} \label{eqn:big3} \\&\quad+ \lambda (|T| - |B|)
        \nonumber \\
        &\geq \underbrace{ \pi_0 \E[p^*(x_{-j} ; 0)]{D_{KL}\parens{p^*_j(x_j | x_B; 0) ~\|~ p^*_j(x_j | x_T; 0)}}}_{\eta(T)} + \lambda (|T| - |B|). \label{eqn:big4}
    \end{align}
    where line \ref{eqn:big4} follows from $\E[p^*(x; 0)]{\log \frac{p^*_j(x_j | x_T; 0)}{p_j(x_j | x_{T};  \theta, 0)}} = \E[p^*(x_T)]{D_{KL}\parens{p_j^*(x_j | x_T) \| p_j(x_j | x_{T};  \theta, 0)}} \geq 0$.

    Let's finally define $u = \min\parens{ \curly{\frac{\eta(T)}{|B| - |T|} \mid T \subset \bbrackets{1,d} \backslash\{j\} \text{ and } \eta(T) > 0 \text{ and } |B| > |T|} \cup \{1\}}$ and fix any $\lambda \in ]0, u[$.
    
    Let's assume now that $\psi(T) \geq \psi(B)$ for some $T \subset \bbrackets{1,d} \backslash\{j\} $, and show that we obtain contradictions. 
    
    First, we would have  $\lambda (|B| - |T|) \geq \eta(T)$. In particular we deduce that $|B| \geq |T|$ (since $\eta(T) \geq 0$).

    Now, two possibilities:
    \begin{enumerate}
        \item If $\eta(T) > 0$, then $|B| > |T|$ and by definition of $\lambda$, $\lambda > \lambda$ which is absurd. 
        \item If $\eta(T) = 0$, then $\pi_0 \E[p^*(x_{-j} ; 0)]{D_{KL}\parens{p^*_j(x_j | x_B; 0) ~\|~ p^*_j(x_j | x_T; 0)}} = 0$. This implies that $D_{KL}\parens{p^*_j(x_j | x_B; 0) ~\|~ p^*_j(x_j | x_T; 0)} = 0$ for all $(x_{-j})$; since $p^*(x_{-j} ; 0)$ has positive density and $\pi_0 > 0$.
        Hence, the conditional $p^*_j(x_j | x_B; 0)$ and $p^*_j(x_j | x_T; 0)$ are identical. Since $B$ was the Markov boundary of $x_j$, that makes $T$ also a Markov blanket of $x_j$. But then, by minimality of the Markov boundary in a faithful graph, we have $B \subset T$. Remember that we had deduced $|B| \geq |T|$. So $B=T$.
    \end{enumerate}

    This ends the proof, where $\lambda \in ]0, u[$.
\end{proof}

\subsection{Proof for Stage 2}
\label{app:proof:stage-two}

Since stage 1 does not remove any true causal parents, theorem 1 of \citet{brouillard2020differentiable} remains valid. 

\subsection{Lemma: Asymptotic Bound on number of edges returned in Stage 1}
\label{app:proof:stage-one-lemma}
We denote the Markov boundary of $j$ in $G^*$ by $\textnormal{bo}_j^{G^*}$, and recall that $\textnormal{bo}_j^{G^*} = \textnormal{pa}^{G^*}_j \cup \textnormal{ch}^{G^*}_j \cup \textnormal{pa}^{G^*}_{\textnormal{ch}^{G^*}_j} \backslash \{j\}.$

The following lemma upper-bounds the theoretical number of edges returned by stage 1 when each node has at most $k$ parents. 

\begin{lemma}
    Assume $G^*$ is sparse such that each node has at most $k$ parents. Then, the total size of all the Markov boundaries is upper-bounded by $dk(k+2) = O(dk^2)$.
\end{lemma}
\begin{proof}
First, note that if each node has at most $k$ parents, then $|E| \leq dk$. Finally,
    \begin{align}
        \sum_{j \in V} |\textnormal{bo}^{G^*}_j| &= \sum_{j \in V} |\textnormal{bo}^{G^*}_j| \\
        &\leq \sum_{j \in V} |\textnormal{pa}^{G^*}_j| + |\textnormal{ch}^{G^*}_j| + |\textnormal{pa}^{G^*}_{\textnormal{ch}^{G^*}_j}| \\
        &\leq |E| + |E| + \sum_{j \in V} \sum_{k \in \textnormal{ch}^{G^*}_j}|\textnormal{pa}^{G^*}_{k}| \\
        &\leq 2kd + \sum_{k \in V} \sum_{j \in \textnormal{pa}^{G^*}_j}|\textnormal{pa}^{G^*}_{k}| \\
        &\leq 2kd + dk^2
    \end{align}
\end{proof}

\newpage
\section{Methods}
\label{app:sec:methods}
\subsection{Model Details}
\label{app:sdcd-model}

In SDCD, the conditional distributions, $p_j(x_j | x_{-j}; \theta, k)$, are modeled as Gaussian distributions where the mean and variance are learned by a neural network that takes in all of the other $x_{-j}$ as input. The initial layer of the network applies $d$ independent linear transformations followed by a sigmoid nonlinearity to the input and outputs $d$ hidden states of size 10. Each of the $d$ hidden states corresponds to the features then used to predict each variable. Each hidden state is fed into two linear layers: one to predict the mean parameter of the conditional and one to predict the variance parameter of the conditional. For the variance, a softplus operation is applied to the output of the linear layer to constrain the variance to be strictly positive. 

\subsection{Algorithm Details}
\label{app:sdcd-algo}

\subsubsection{Spectral Acyclicity Constraint Estimation}
As described in \Cref{th:constraint-spectral-valid}, the gradient of the spectral acyclicity constraint can be computed as $h_{\rho} ( A ) = v_d  u_d^\top / v_d^\top  u_d$, where $u_d, v_d$ are the right and left eigenvectors of $A$ respectively. Using the power iteration method, which involves a fixed number of matrix-vector multiplications, $u_d,v_d$ can be estimated in $O(d^2)$. Specifically, the updates are as follows:
$$
u_d^{(i+1)} := \frac{A^\top u_d^{(i)}}{\|u_d\|_2}, \; v_d^{(i+1)} := \frac{A v_d^{(i)}}{\|v_d\|_2} 
$$
where $u_d,v_d$ are initialized as $u_d^{(1)}, v_d^{(1)} := [\frac{1}{\sqrt{d}}, \dots, \frac{1}{\sqrt{d}}]$ at the very first epoch of SDCD. In our implementation, we use 15 iterations to estimate the spectral acyclicity constraint value.

Importantly, we re-use the estimates of $u_d$ and $v_d$ from one epoch to another, as we don't expect $A$ (and its eigenvectors) to change drastically. 

Hence, at each epoch, we initialize $u_d, v_d$ using their last epoch's value and perform 15 power iterations.

\subsubsection{SDCD Algorithm}
The SDCD algorithm follows a two-stage procedure. In the first stage, the coefficient of the spectral acyclicity constraint, $\gamma$, is fixed at zero. We use an Adam optimizer with a learning rate, $\eta_1$, specific to stage 1 to perform minibatch gradient-based optimization. The coefficients corresponding to the L1 and L2 penalties, $\alpha_1$ and $\beta_1$, respectively, are fixed throughout training. The stage 1 training loss is written as:
\begin{align*}
\mathcal{L}_1(X, \theta, \alpha_1, \beta_1) &= S_{\alpha_1,\beta_1}(\theta) \\
&= \frac{1}{n} \sum\limits_{i=1}^n \log p(x^i ; \theta, t^i) - \alpha_1 \|A_\theta\|_1 - \beta_1 \|\theta\|_2^2.
\end{align*}

To prevent the model from learning implicit self-loops, the weights corresponding to the predicted variable are masked out for every hidden state output by the initial neural network layer. Thus, the prediction of each variable is prevented from being a function of the same variable.

In interventional regimes, the log-likelihood terms corresponding to the prediction of intervened variables are zeroed out. The intervened variables do not have to be modeled as we assume perfect interventions.

Stage 1 is run for a fixed number of epochs. By default, stage 1 also has an early stopping mechanism that uses the reconstruction loss of a held-out validation set of data (sampled uniformly at random from the training set) as the early stopping metric. If the validation reconstruction loss does not achieve a new minimum after a given number of epochs, the stage 1 training loop is exited.

At the end of stage 1, the learned input layer weights are used to compute a set of removed edges, $\hat{R}$, for stage 2. Let $W \in \mathbb{R}^{d \times d \times 10}$ represent the input layer weights. Then, each value of the implicitly defined weighted adjacency matrix is computed as the L2 vector norm for the corresponding set of weights (i.e., $A_{\theta, i,j}:= \|W_{i,j,:}\|_2$). This weighted adjacency matrix is discretized with a fixed threshold, $\tau_1$, such that each edge, $(i,j)$, is removed if it falls below the threshold (i.e., $A_{\theta, i,j} < \tau_1$).

In stage 2, the spectral acyclicity constraint is introduced. Like stage 1, we use an Adam optimizer with learning rate, $\eta_2$, and perform minibatch gradient-based optimization. Once again, the L1 and L2 coefficients, $\alpha_2, \beta_2$, are fixed throughout training. Rather than a fixed $\gamma$, SDCD takes an increment value, $\gamma^+ \in \mathbb{R}^+$, determining the rate at which $\gamma$ increases every epoch. The training loss for stage 2 is as follows:

\begin{align*}
\mathcal{L}_2(X, \theta, \hat{R}, \alpha_2, \beta_2, \gamma) &= S_{\alpha_2, \beta_2}(\theta) - \gamma h_{\rho}(A_\theta) \\
&= \frac{1}{n} \sum\limits_{i=1}^n \log p(x^i ; \theta, t^i) - \alpha_2 \|A_\theta\|_1 - \beta_2 \|\theta\|_2^2 - \gamma h_\rho (A_\theta).
\end{align*}

The same masking strategy as in stage 1 is used to prevent self-loops in $A_\theta$. However, the input layer weights corresponding to edges $(i,j) \in \hat{R}$ are also masked.

Like before, the reconstruction loss terms corresponding to intervened variables are removed from the loss.

To reduce the sensitivity of stage 2 to the choice of $\gamma^+$ and to prevent the acyclicity constraint term from dominating the loss, the linear increment schedule is frozen when $A_\theta$ achieves a DAG at the final threshold, $\tau_2$. In practice, the DAG check is performed every 20 epochs. If the adjacency matrix returns to being cyclic throughout training, the $\gamma$ increment schedule restarts to increase from where it left off.

The early stopping metric is computed similarly to stage 1, but in stage 2, the early stopping can only kick in when $\gamma$ has been frozen. If the $\gamma$  schedule is resumed due to $A_\theta$ reintroducing a cycle, the early stopping is reset.

Lastly, once stage 2 is complete, $A_\theta$ is computed and thresholded according to a fixed threshold, $\tau_2$. All values exceeding the threshold (i.e., $A_\theta, i,j \ge \tau_2$) are considered edges in the final graph prediction.

The thresholded adjacency matrix may contain cycles if stage 2 runs to completion without hitting early stopping. To ensure a DAG, we follow a greedy edge selection procedure detailed in \Cref{alg:dag-trim}.

Pseudocode for a simplified SDCD algorithm (excludes $\gamma$ freezing and early stopping) is provided in \Cref{alg:sdcd}.

\subsubsection{Time and Space Complexity}
\label{app:sec:time_space_complexity}
The time complexity of each iteration of SDCD is $O(d^2)$. The forward pass in stage 1 can be computed in $O(d^2)$ time. On the other hand, each of the $d$ prediction problems can be computed independently. This allows for parallelizing the $d$ problems, each taking $O(d)$ time. Stage 2 also takes $O(d^2)$ time as the spectral acyclicity constraint and the forward pass both take $O(d^2)$ time to compute. Thus, the time complexity of each iteration in both stages is $O(d^2)$.

If the sparsity pattern of the underlying causal graph is known beforehand such that each variable has at most $k$ parents, we can further tighten the time complexity of SDCD. By \Cref{app:proof:stage-one-lemma}, we know the size of the set of remaining edges after stage 1 is $O(dk^2)$. Using sparse matrix multiplication, the spectral acyclicity constraint can be done in $O(dk^2)$, which is effectively linear in $d$ if $k\ll d$. However, this improvement only becomes significant when $d>10,000$ (from experiments not reported in this paper). 

The space complexity of the algorithm is $O(d^2)$, as the number of parameters in the input layer scale quadratically in the number of features.

\begin{algorithm}
\caption{SDCD}\label{alg:sdcd}
\begin{algorithmic}
\REQUIRE $\alpha_1, \alpha_2 \in \mathbb{R}^+, \beta_1, \beta_2 \in \mathbb{R}^+, \gamma^+ \in \mathbb{R}^+$, \\ \;\;\;\;\;\;\;\;\;\;\;\;\; $\tau_1, \tau_2 \in \mathbb{R}^+, \eta_1, \eta_2 \in \mathbb{R}^+, E_{1}, E_{2} \in \mathbb{Z}^+$
\STATE $A_\theta^{(0)} \gets \vec{0}^{G\times G}$
\STATE $\theta^{(0)}_{-A_\theta} \gets $ RandomGaussianInit()
\STATE $e \gets 0$
\WHILE{$e < E_{1}$}
    \STATE $\theta^{(e + 1)} := \mathrm{AdamUpdate}(\theta^{(e)}, \nabla \mathcal{L}_1(X, \theta^{(e)}, \alpha_1, \beta_1), \eta_1)$
    \STATE $e \gets e + 1$
\ENDWHILE

\STATE $\hat{R} := \mathrm{Threshold}(A_\theta^{(E_1)}, \tau_2)$

\STATE $A_\theta^{(E_1)} \gets \vec{0}^{G\times G}$
\STATE $\theta^{(E_1)}_{-A_\theta} \gets $ RandomGaussianInit()

\STATE $\gamma \gets 0$
\WHILE{$e < E_1 + E_{2}$}
    \STATE $\theta^{(e + 1)} := \mathrm{AdamUpdate}(\theta^{(e)}, \nabla \mathcal{L}_1(X, \theta^{(e)}, \hat{R}, \alpha_2, \beta_2, \gamma), \eta_2)$
    \STATE $\gamma \gets \gamma + \gamma^+$
    \STATE $e \gets e + 1$
\ENDWHILE
\\
\OUTPUT $\mathrm{DAGTrim}(A_\theta^{(E_{2})}, \tau_2)$

\end{algorithmic}
\end{algorithm}
\begin{algorithm}
\caption{DAGTrim}\label{alg:dag-trim}
\begin{algorithmic}
\REQUIRE $A_\theta \in \mathbb{R}^{D \times D}, \tau \in \mathbb{R}^+$
\STATE $E \leftarrow \emptyset$ \COMMENT{Initialize the set of final edges.}
\STATE $C \leftarrow [(i,j) \in \bbrackets{1,d}^2 \mid (A_{\theta, i,j} > \tau ]$ \COMMENT{Candidate edges above threshold $\tau$.} 
\STATE Sort $C$ by decreasing $A_{\theta, i, j}$.
\FOR{each $(i,j) \in C$}
\IF{the graph with edges $E \cup \{(i,j)\}$ is still acyclic}
\STATE $E \leftarrow E \cup \{(i,j)\}$ \COMMENT{We add the edge if it does not create a cycle.}
\ENDIF
\ENDFOR
\end{algorithmic}
\end{algorithm}

\newpage
\section{Empirical Studies Details}

\subsection{Simulation Details}
\label{app:empirical:simulations-details}
To judge the performance of SDCD against existing methods over both interventional and observational data, we generated simulated data according to the following procedure:
\begin{itemize}
    \item Draw a random undirected graph from the Erd\H{o}s-R\'enyi distribution.
    \item Convert the undirected graph into a DAG $G^*$ by setting the direction of each edge $i \to j$ if $\pi(i) < \pi(j)$, where $\pi$ is a random permutation of the nodes.
    \item Form $d$ possible sets of interventions that target one variable at a time: $I_j = \curly{j}$ and $I_0 = \emptyset$.
    \item Draw a set of random fully connected neural networks $\operatorname{MLP}^{(j)} : \R^{|\pa^{G^*}_j|} \rightarrow \R^{100} \rightarrow 1$, each one with one 100-dimensional hidden layer. Each neural network parametrizes the mean of the observational conditional distributions: $$p^*_j(x_j \mid x_{\pa^{G^*}_j}; 0) \sim \mathcal{N}\parens{\mu = \operatorname{MLP}^{(j)}(x_{\pa^{G^*}_j}), \sigma=0.5 }.$$
    \item For intervention distribution $k \geq 1$, perform a hard intervention on variable $k$ and set $$p^{*}_j(x_{k}; k) \sim \mathcal{N}(0, 0.1). $$
    \item Draw the data according to the model, with 10,000 observational samples and 500 extra interventional samples per target variable.
    \item Standardize the data.
\end{itemize}


We consider several values of $d$ to simulate different scenarios.

\subsection{Choice of Hyperparameters}
\label{app:empirical:hyperparameters}

We fixed the hyperparameters as follows: $\alpha_1 := 1\mathrm{e}{-2}, \beta_1 := 2\mathrm{e}{-4}, \eta_1 := 2\mathrm{e}{-3}, \tau_1 := 0.2, \alpha_2 := 5\mathrm{e}{-4}, \beta_2 := 5\mathrm{e}{-3}, \eta_2 := 1\mathrm{e}{-3}, \gamma^+ := 0.005, \tau_2 := 0.1$. We found that these selections worked well empirically across multiple simulated datasets and were used in all experiments without simulation-specific fine-tuning.

Each stage was run for 2000 epochs with a batch size of 256, and the validation loss was computed over a held-out fraction of the training dataset (20\% of the data) every 20 epochs for early stopping. In stage 2, the DAG check of the implicit adjacency matrix was performed every 20 epochs before the validation loss computation.

\subsection{Baseline Methods}
\label{app:empirical:benchmark}
Here, we provide details on the baseline methods and cite which implementations were used for the experiments. For DCDI and DCDFG, we used the implementations from \url{https://github.com/Genentech/dcdfg}, using the default parameters for optimization. For DCDFG, we used 10 modules in our benchmarks, as reported in the paper experiments. For GIES, we used the Python implementation from \url{https://github.com/juangamella/gies}, using the default parameters. For DAGMA, we used the original implementation from \url{https://github.com/kevinsbello/dagma} with the default parameters. For NOTEARS, we used the implementation from \url{https://github.com/xunzheng/notears}, and for NOBEARS, we used the implementation from \url{https://github.com/howchihlee/BNGPU}. For NOTEARS and NOBEARS, we found the default thresholds for determining the final adjacency matrix performed poorly or did not return a DAG, so for each of these baselines, we followed the same procedure described in \citet{lopez2022large}: we find the threshold that returns the largest possible DAG via binary search. sortnregress \citep{reisach2021beware} is a trivial baseline meant to ensure that the causal graph cannot be easily inferred from the variance pattern across the variables. For this baseline, we used the implementation in \url{https://github.com/Scriddie/Varsortability}.

\subsection{Robustness Checks}
\label{app:empirical:robustness}

Below, we discuss three categories of issues that commonly arise when evaluating causal discovery methods and address each issue with a diagnostic metric.

\parhead{Sparsity}
Particularly when the true causal graph is sparse, SHD may favor sparser predictions since, in the extreme case, the empty graph achieves an SHD equal to the number of true edges. To show the relative performance of the benchmarked methods with respect to this trivial solution, we indicate the number of true edges for each simulated setting in \Cref{fig:observational} and \Cref{fig:interventional}. We find that most methods outperform this baseline. Additionally, we report the F1 score and the recall of the predictions (see \Cref{app:fig:empirical:observational_recall,app:fig:empirical:observational-f1-cpdag}), two metrics that suffer when a method predicts many false negatives. We find that SDCD still outperforms other methods with these metrics.

\parhead{Identifiability} 
In settings with incomplete or no interventional data, the true causal graph may be unidentifiable, meaning multiple $\mathcal{I}$-Markov equivalent graphs can maximize the score \citep{brouillard2020differentiable}. Therefore, graphs in the same Markov equivalence class as the true causal graph may have positive SHD values despite being optimal with respect to the available data. As proposed in \citet{peters2014causal}, we also compute an adapted version of the SHD to compare the Markov equivalence class of the methods' results against the true Markov equivalence class instead of the graphs themselves. This metric, called SHD-CPDAG, is computed as the SHD between the completed partially directed acyclic graph (CPDAG) of the predicted graph and the CPDAG of the true graph. Unlike the regular SHD metric, this metric is zero if two graphs are in the same equivalence class.
We report it alongside SHD for our experiments in \Cref{app:fig:empirical:observational-f1-cpdag} to better represent the results in scenarios with an unidentifiable causal graph. We find very similar results.

\parhead{Simulation issues} 
As discussed in \citet{reisach2021beware}, certain simulation processes used for causal discovery benchmarking exhibit an issue where the order of the variables, when sorted by sample variance, reflects the true causal ordering of the graph. As a result, methods that exploit this phenomenon to accurately infer the causal graph may be misrepresented. To ensure that our simulation process does not suffer from this issue and that the methods are being properly evaluated, we take two complementary steps recommended in \citet{reisach2021beware}: 
(i) we standardize the data before being input into any of the evaluated methods so that no artificial sample variance information can be exploited, and (2) we include the trivial baseline, sortnregress \citep{reisach2021beware}, which is designed to exploit sample variance artifacts from a flawed simulation, and should be outperformed by an effective, scale-invariant algorithm. We find that sortnregress performs poorly, which confirms that our normalization scheme removes simulation artifacts, and we find that SDCD and its competing methods beat sortnregress by a wide margin.

\newpage
\section{Supplementary Figures and Tables}
\label{app:sec:figs-and-tables}

\label{app:sec:constraint-optimization-behaviors}
\begin{figure}[h]
    \centering
    \includegraphics[width=0.7\linewidth]{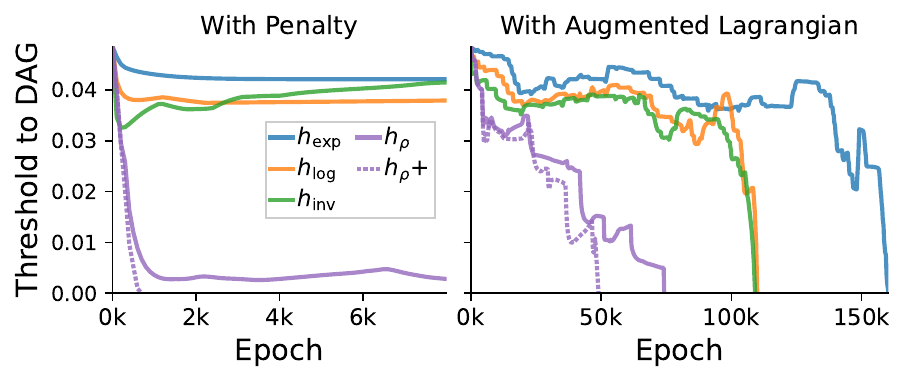}
    \caption{The effect of constraints on the learned graph throughout training. The training with penalty $h_\rho+$ (dashed purple, exactly $h_\rho$ with a hard mask on the diagonal as to prevent self-loops, as implemented in SDCD) converges the fastest toward a DAG. (left) training with $h$ as a regularization penalty. (right) training with $h$ as an augmented Lagrangian constraint. \textit{Threshold to DAG} is the smallest $\eta$ at which all edges with weight $> \eta$ form a DAG.}
    \label{app:fig:constraint-optimization-behaviors}
\end{figure}

\begin{table}[h]
    \centering
    \resizebox{0.6\columnwidth}{!}{%
    \begin{tabular}{llC{1.2cm}llll}
\toprule
s & d & $\delta$ & Method & SDCD & DCDI-G & DCDI-DSF \\
\midrule
\multirow[t]{6}{*}{1} & \multirow[t]{3}{*}{10} & \multirow[t]{3}{*}{22.2\%} & L & $\mathbf{0.7{\scriptstyle \pm1.2}}$ & $1.3{\scriptstyle \pm1.9}$ & $0.9{\scriptstyle \pm1.3}$ \\
 &  & & NL-Add & $\mathbf{0.6{\scriptstyle \pm0.7}}$ & $5.2{\scriptstyle \pm7.5}$ & $4.2{\scriptstyle \pm5.6}$ \\
 &  & & NL-NN & $\mathbf{0.7{\scriptstyle \pm0.7}}$ & $2.3{\scriptstyle \pm3.6}$ & $7.0{\scriptstyle \pm10.7}$ \\
\midrule
 & \multirow[t]{3}{*}{20} & \multirow[t]{3}{*}{10.5\%} & L & $\mathbf{1.4{\scriptstyle \pm3.4}}$ & $5.4{\scriptstyle \pm4.5}$ & $3.6{\scriptstyle \pm2.7}$ \\
 &  & & NL-Add & $\mathbf{4.1{\scriptstyle \pm3.0}}$ & $21.8{\scriptstyle \pm30.1}$ & $4.3{\scriptstyle \pm1.9}$ \\
 &  & & NL-NN & $\mathbf{3.0{\scriptstyle \pm2.5}}$ & $13.9{\scriptstyle \pm20.3}$ & $8.3{\scriptstyle \pm4.1}$ \\
\midrule
\multirow[t]{6}{*}{4} & \multirow[t]{3}{*}{10} & \multirow[t]{3}{*}{88.9\%} & L & $5.2{\scriptstyle \pm3.5}$ & $\mathbf{3.3{\scriptstyle \pm2.1}}$ & $3.7{\scriptstyle \pm2.3}$ \\
 &  & & NL-Add & $4.8{\scriptstyle \pm2.1}$ & $\mathbf{4.3{\scriptstyle \pm2.4}}$ & $5.5{\scriptstyle \pm2.4}$ \\
 &  & & NL-NN & $7.3{\scriptstyle \pm3.0}$ & $2.4{\scriptstyle \pm1.6}$ & $\mathbf{1.6{\scriptstyle \pm1.6}}$ \\
\midrule
 & \multirow[t]{3}{*}{20} & \multirow[t]{3}{*}{42.1\%} & L & $18.8{\scriptstyle \pm10.5}$ & $23.7{\scriptstyle \pm5.6}$ & $\mathbf{16.6{\scriptstyle \pm6.4}}$ \\
 &  & & NL-Add & $\mathbf{18.0{\scriptstyle \pm7.3}}$ & $35.2{\scriptstyle \pm13.2}$ & $26.7{\scriptstyle \pm16.9}$ \\
 &  & & NL-NN & $14.9{\scriptstyle \pm1.9}$ & $16.8{\scriptstyle \pm8.7}$ & $\mathbf{11.8{\scriptstyle \pm2.1}}$ \\
\bottomrule
\end{tabular}
}
    \caption{Means and standard deviations of SHD scores over simulations from \citet{brouillard2020differentiable}. The ``Method'' column refers to the model used to simulate the causal relationships. ``L'' refers to linear model, ``NL-Add'' refers to nonlinear, additive model, and ``NL-NN'' refers to nonlinear, non-additive (neural network) model. We refer to \citet{brouillard2020differentiable} for the simulation details. The results are reported alongside the values presented in the original paper. $s$ refers to the expected number of edges per node, $d$ denotes the number of nodes, and the edge density, $\delta$, is computed as the fraction of $\frac{d(d-1)}{2}$, the maximum number of edges for a DAG. The lowest average SHD values are set in bold.}
    \label{app:tab:empirical:dcdi}
\end{table} 
\clearpage

\label{app:fig:empirical:observational}

\begin{figure}[h]
    \centering
    \includegraphics[width=0.7\linewidth]{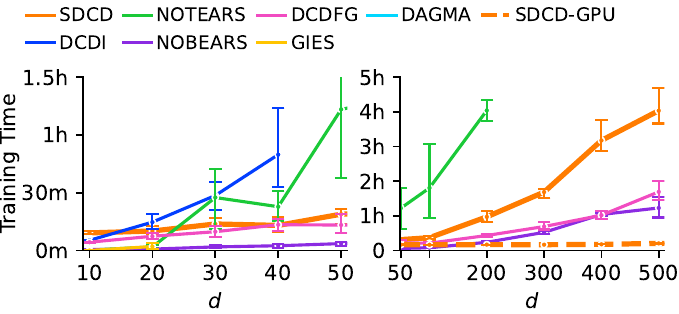}
    \caption{Training runtimes across simulations from \Cref{fig:observational}. SDCD on GPU (dashed) scales to 500 variables in under 20 minutes. Error bars indicate std on 5 random datasets for $d<50$ and 3 random datasets for $d\ge 50$.}
    \label{app:fig:empirical:runtime}
\end{figure}

{
\renewcommand{\arraystretch}{1}
\begin{table}[h]
    \centering
\resizebox{\columnwidth}{!}{%
\begin{tabular}{lllllllllllll}
\toprule
& SDCD & SDCD-GPU & DCDI & DCDFG & GIES & DAGMA & NOTEARS & NOBEARS & SCORE & sortnregress & AVICI & NOCURL \\
d & & & & & & & & & & & & \\
\midrule
10 & 14.7 ${\scriptstyle \pm 5.5}$ & NT & 24.3 ${\scriptstyle \pm 3.9}$ & 24.6 ${\scriptstyle \pm 6.0}$ & 27.8 ${\scriptstyle \pm 3.9}$ & 25.3 ${\scriptstyle \pm 6.2}$ & 35.3 ${\scriptstyle \pm 1.9}$ & 33.5 ${\scriptstyle \pm 3.0}$ & 14.4 ${\scriptstyle \pm 4.0}$ & 28.6 ${\scriptstyle \pm 4.7}$ & 23.52 & 33.23 \\
20 & 35.7 ${\scriptstyle \pm 6.2}$ & NT & 31.7 ${\scriptstyle \pm 6.5}$ & 108.0 ${\scriptstyle \pm 14.3}$ & 123.4 ${\scriptstyle \pm 12.3}$ & 62.0 ${\scriptstyle \pm 12.0}$ & 75.4 ${\scriptstyle \pm 4.0}$ & 74.2 ${\scriptstyle \pm 3.0}$ & 118.6 ${\scriptstyle \pm 8.5}$ & 81.6 ${\scriptstyle \pm 6.3}$ & 60.93 & 82.37 \\
30 & 53.8 ${\scriptstyle \pm 11.9}$ & NT & 55.5 ${\scriptstyle \pm 10.7}$ & 258.8 ${\scriptstyle \pm 32.6}$ & NA & 89.4 ${\scriptstyle \pm 10.9}$ & 113.1 ${\scriptstyle \pm 4.8}$ & 113.7 ${\scriptstyle \pm 3.3}$ & 275.9 ${\scriptstyle \pm 21.0}$ & 134.6 ${\scriptstyle \pm 8.4}$ & 97.13 & 134.60 \\
40 & 64.0 ${\scriptstyle \pm 13.7}$ & NT & 102.4 ${\scriptstyle \pm 21.6}$ & 426.6 ${\scriptstyle \pm 73.7}$ & NA & 115.3 ${\scriptstyle \pm 13.4}$ & 147.9 ${\scriptstyle \pm 6.0}$ & 151.1 ${\scriptstyle \pm 3.4}$ & 454.3 ${\scriptstyle \pm 52.8}$ & 172.9 ${\scriptstyle \pm 12.2}$ & 135.83 & 179.93 \\
50 & 69.9 ${\scriptstyle \pm 12.3}$ & 68.3 ${\scriptstyle \pm 13.3}$ & NA & 660.8 ${\scriptstyle \pm 126.1}$ & NA & NA & 183.4 ${\scriptstyle \pm 7.4}$ & 192.0 ${\scriptstyle \pm 3.5}$ & 619.4 ${\scriptstyle \pm 59.7}$ & 216.6 ${\scriptstyle \pm 12.4}$ & 170.83 & 240.93 \\
100 & 92.7 ${\scriptstyle \pm 9.1}$ & 89.7 ${\scriptstyle \pm 11.0}$ & NA & 1807.3 ${\scriptstyle \pm 788.2}$ & NA & NA & 327.3 ${\scriptstyle \pm 7.5}$ & 389.0 ${\scriptstyle \pm 3.6}$ & NA & 421.3 ${\scriptstyle \pm 12.0}$ & 366.50 & 513.07 \\
200 & 225.3 ${\scriptstyle \pm 13.7}$ & 228.0 ${\scriptstyle \pm 18.3}$ & NA & 5657.3 ${\scriptstyle \pm 2982.6}$ & NA & NA & 619.0 ${\scriptstyle \pm 4.2}$ & 770.0 ${\scriptstyle \pm 7.8}$ & NA & 824.0 ${\scriptstyle \pm 19.0}$ & NT & NT \\
300 & 350.0 ${\scriptstyle \pm 12.5}$ & 360.0 ${\scriptstyle \pm nan}$ & NA & 7284.7 ${\scriptstyle \pm 5072.3}$ & NA & NA & NA & 1149.0 ${\scriptstyle \pm 14.0}$ & NA & 1190.7 ${\scriptstyle \pm 26.3}$ & NT & NT \\
400 & 466.3 ${\scriptstyle \pm 62.4}$ & 471.7 ${\scriptstyle \pm 68.0}$ & NA & 3779.7 ${\scriptstyle \pm 507.3}$ & NA & NA & NA & 1534.7 ${\scriptstyle \pm 3.1}$ & NA & 1585.0 ${\scriptstyle \pm 59.3}$ & NT & NT \\
500 & 621.7 ${\scriptstyle \pm 10.7}$ & 621.0 ${\scriptstyle \pm 10.5}$ & NA & 7252.7 ${\scriptstyle \pm 3284.6}$ & NA & NA & NA & 1915.7 ${\scriptstyle \pm 18.8}$ & NA & 1974.3 ${\scriptstyle \pm 34.6}$ & NT & NT \\
\bottomrule
\end{tabular}}

    \caption{Detailed results of SHD means and standard deviations from \Cref{fig:observational}. SDCD-GPU was only run for $d\ge 50$. All other \texttt{NA} values correspond to failed runs (possibly from timeout after 6h, e.g., GIES, or from training error, e.g., DCDI). \texttt{NT} corresponds to the method not having been tested on that particular example.}
    \label{app:tab:empirical:observational}
\end{table}
}

\begin{figure}[h!]
    \centering
    \includegraphics[width=0.9\linewidth]{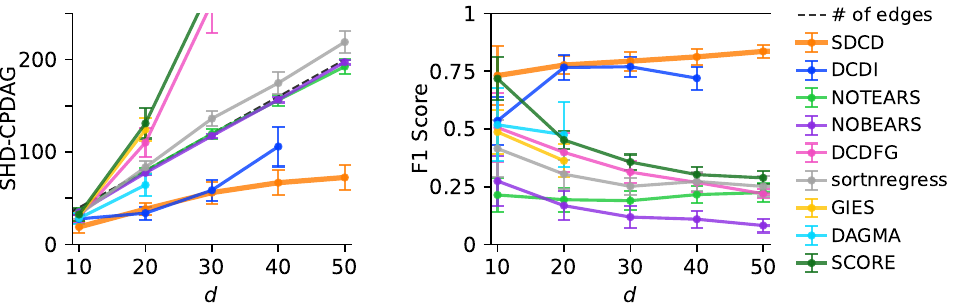}
    \caption{F1 and SHD-CPDAG metrics across simulations from \Cref{fig:observational}, observational data with increasing numbers of variables $d$. Missing data points imply the method failed to run.  Error bars indicate std on 30 random datasets.}
    \label{app:fig:empirical:observational-f1-cpdag}
\end{figure}

\clearpage

\begin{figure}[h]
    \centering
    \includegraphics[width=0.7\linewidth]{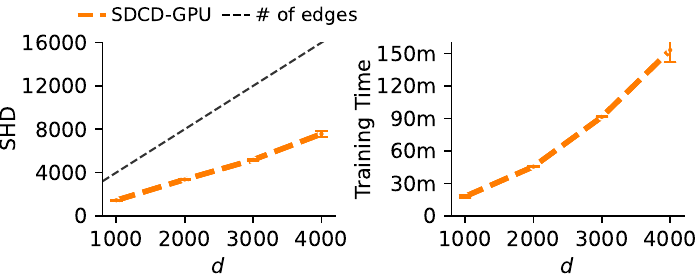}
    \caption{SDCD on GPU (dashed) scales to 4000 variables under 3 hours while maintaining competitive SHD. Error bars indicate std on 3 random datasets for $d=1000, 2000$ and 2 random datasets for $d=3000, 4000$.}
    \label{app:fig:empirical:observational-large}
\end{figure}

\begin{table}[h]
    \centering 
    \resizebox{0.25\columnwidth}{!}{
\begin{tabular}{ll}
\toprule
d &                           SDCD-GPU \\
\midrule
1000 &   1438.7 ${\scriptstyle \pm 59.2}$ \\
2000 &   3356.7 ${\scriptstyle \pm 70.0}$ \\
3000 &   5172.5 ${\scriptstyle \pm 89.8}$ \\
4000 &  7567.0 ${\scriptstyle \pm 343.7}$ \\
\bottomrule
\end{tabular}
}
    \caption{Detailed results of SHD means and standard deviations from \Cref{app:fig:empirical:observational-large}.}
    \label{app:tab:empirical:observational-large}
\end{table}

\clearpage

In addition to SHD, we computed precision and recall metrics over the predicted edges with respect to the true edges for both observational and interventional scenarios. The precision is the fraction of true edges among all the predicted edges. The recall is the fraction of true edges that have been correctly predicted.

\begin{figure}[h!]
    \centering
    \includegraphics[width=.7\linewidth]{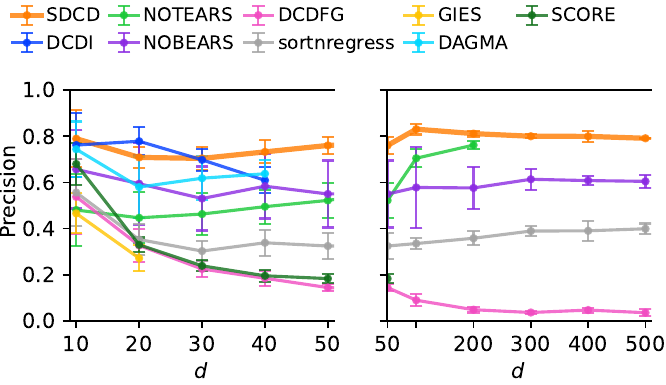}
    \caption{Precision across simulations from \Cref{fig:observational}, observational data with increasing numbers of variables $d$. The SDCD(-CPU) and SDCD-GPU lines overlap, indicating consistent results. Missing data points imply the method failed to run.  Error bars indicate std on 30 random datasets for $d<50$ and 5 random datasets for $d\ge 50$.}
    \label{app:fig:empirical:observational_precision}
\end{figure}

\begin{figure}[h!]
    \centering
    \includegraphics[width=.7\linewidth]{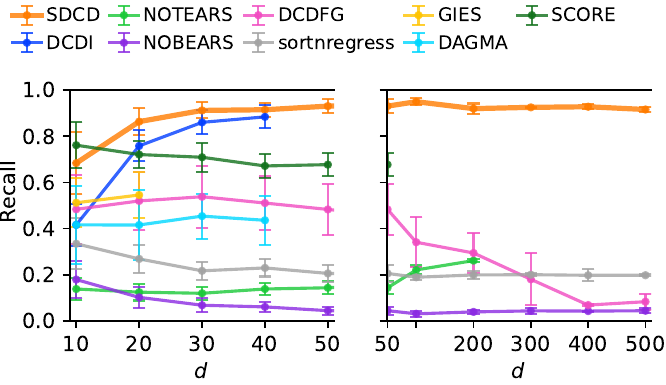}
    \caption{Recall across simulations from \Cref{fig:observational}, observational data with increasing numbers of variables $d$. The SDCD(-CPU) and SDCD-GPU lines overlap, indicating consistent results. Missing data points imply the method failed to run.  Error bars indicate std on 30 random datasets for $d<50$ and 5 random datasets for $d\ge 50$.}
    \label{app:fig:empirical:observational_recall}
\end{figure}
\clearpage

\begin{figure}[p]
    \centering\includegraphics[width=0.7\linewidth]{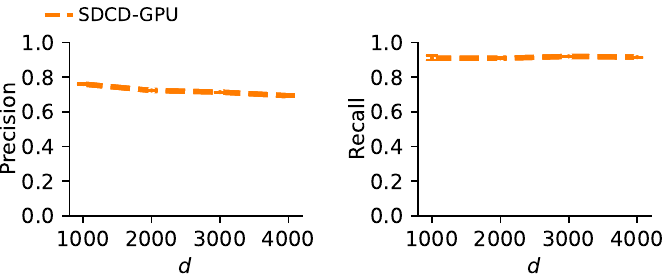}
    \caption{Precision and recall across simulations from \Cref{app:fig:empirical:observational-large}, observational data with increasing numbers of variables $d$. Error bars indicate std on 3 random datasets for $d=1000, 2000$ and 2 random datasets for $d=3000, 4000$.}
    \label{app:fig:empirical:observational-large-precision-recall}
\end{figure}
\clearpage

\begin{figure}[p]
    \centering
    \includegraphics[width=\linewidth]{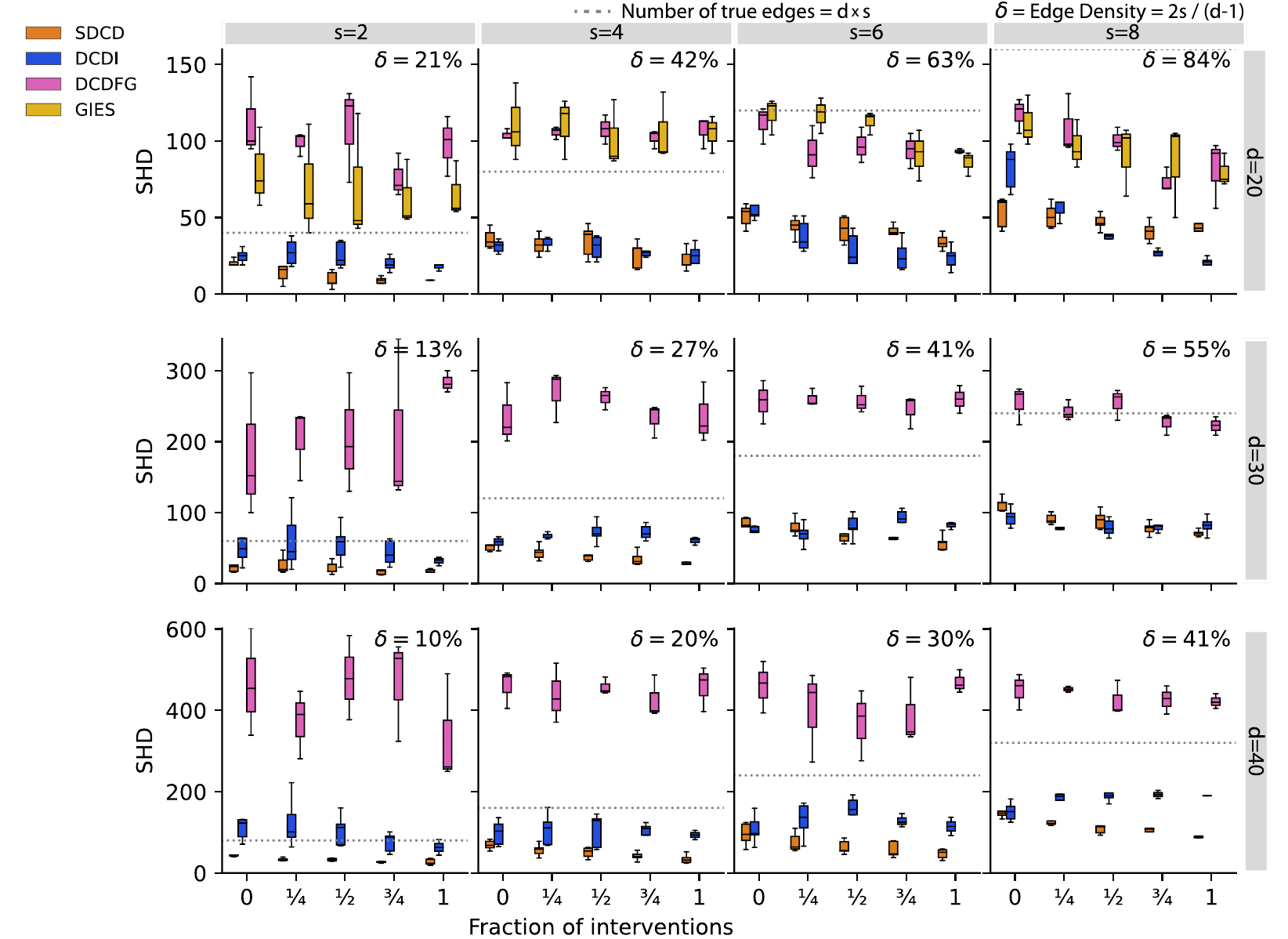}
    \caption{SHD across simulations with an increasing proportion of variables intervened on, varying the total number of variables $d$ (columns) and average edges per variable $s$ (rows). Extended version of \Cref{fig:interventional} with DCDFG and GIES and $s=8$.
    Boxplots over 5 random datasets.}
    \label{app:fig:empirical:interventional_full}
\end{figure}

\clearpage
{
\small
\begin{table}[p]
    \centering
    \resizebox{0.65\columnwidth}{!}{%
\begin{tabular}{lllC{3.5cm}llll}
\toprule
  s & d & $\delta$ & Fraction of Variables Intervened on &                             SDCD &                             DCDI &                             DCDFG &                             GIES \\
\midrule
  &    &      & 0.00 &    \textbf{18.0 ${\scriptstyle \pm 6.5}$} &    24.8 ${\scriptstyle \pm 4.6}$ &   112.3 ${\scriptstyle \pm 25.8}$ &   80.3 ${\scriptstyle \pm 26.1}$ \\
  &    &      & 0.25 &    \textbf{13.4 ${\scriptstyle \pm 5.7}$} &    27.4 ${\scriptstyle \pm 8.6}$ &     99.0 ${\scriptstyle \pm 7.8}$ &   70.0 ${\scriptstyle \pm 36.8}$ \\
2 & 20 & 21\% & 0.50 &     \textbf{9.4 ${\scriptstyle \pm 5.4}$} &    25.4 ${\scriptstyle \pm 8.5}$ &   109.0 ${\scriptstyle \pm 31.4}$ &   69.7 ${\scriptstyle \pm 41.9}$ \\
  &    &      & 0.75 &     \textbf{9.0 ${\scriptstyle \pm 2.1}$} &    19.8 ${\scriptstyle \pm 4.8}$ &    76.0 ${\scriptstyle \pm 14.2}$ &   62.7 ${\scriptstyle \pm 22.0}$ \\
  &    &      & 1.00 &     \textbf{9.0 ${\scriptstyle \pm 2.8}$} &    18.8 ${\scriptstyle \pm 3.3}$ &    98.0 ${\scriptstyle \pm 19.7}$ &   65.7 ${\scriptstyle \pm 18.5}$ \\
  \cline{2-8}
  &    &      & 0.00 &     \textbf{24.6 ${\scriptstyle \pm 9.6}$} &   56.0 ${\scriptstyle \pm 32.9}$ &  183.0 ${\scriptstyle \pm 102.1}$ &                               NA \\
  &    &      & 0.25 &   \textbf{26.8 ${\scriptstyle \pm 13.1}$} &   60.4 ${\scriptstyle \pm 40.9}$ &   204.3 ${\scriptstyle \pm 51.4}$ &                               NA \\
  & 30 & 13\% & 0.50 &    \textbf{21.8 ${\scriptstyle \pm 9.0}$} &   56.2 ${\scriptstyle \pm 26.6}$ &   206.7 ${\scriptstyle \pm 84.3}$ &                               NA \\
  &    &      & 0.75 &    \textbf{16.2 ${\scriptstyle \pm 3.4}$} &   43.2 ${\scriptstyle \pm 17.8}$ &  207.0 ${\scriptstyle \pm 119.7}$ &                               NA \\
  &    &      & 1.00 &    \textbf{18.2 ${\scriptstyle \pm 2.2}$} &    31.7 ${\scriptstyle \pm 6.1}$ &   283.7 ${\scriptstyle \pm 15.2}$ &                               NA \\
  \cline{2-8}
  &    &      & 0.00  &    \textbf{44.2 ${\scriptstyle \pm 4.1}$} &  109.2 ${\scriptstyle \pm 27.6}$ &  465.0 ${\scriptstyle \pm 131.8}$ &                               NA \\
  &    &      & 0.25 &    \textbf{33.4 ${\scriptstyle \pm 3.6}$} &  123.8 ${\scriptstyle \pm 62.1}$ &   372.7 ${\scriptstyle \pm 84.3}$ &                               NA \\
  & 40 & 10\% & 0.50 &    \textbf{33.0 ${\scriptstyle \pm 3.5}$} &  105.6 ${\scriptstyle \pm 38.8}$ &  479.7 ${\scriptstyle \pm 103.5}$ &                               NA \\
  &    &      & 0.75 &    \textbf{27.6 ${\scriptstyle \pm 3.4}$} &   76.0 ${\scriptstyle \pm 24.4}$ &  469.3 ${\scriptstyle \pm 126.6}$ &                               NA \\
  &    &      & 1.00 &    \textbf{27.0 ${\scriptstyle \pm 7.5}$} &   63.0 ${\scriptstyle \pm 26.9}$ &  333.7 ${\scriptstyle \pm 135.5}$ &                               NA \\
\midrule
  &    &      & 0.00 &    36.0 ${\scriptstyle \pm 6.4}$ &    \textbf{31.2 ${\scriptstyle \pm 4.1}$} &    104.0 ${\scriptstyle \pm 3.5}$ &  110.7 ${\scriptstyle \pm 25.3}$ \\
  &    &      & 0.25 &    \textbf{32.2 ${\scriptstyle \pm 6.6}$} &    33.0 ${\scriptstyle \pm 3.6}$ &    105.7 ${\scriptstyle \pm 4.2}$ &  110.7 ${\scriptstyle \pm 20.0}$ \\
4 & 20 & 42\% & 0.50 &   34.6 ${\scriptstyle \pm 10.6}$ &    \textbf{30.4 ${\scriptstyle \pm 7.6}$} &    107.7 ${\scriptstyle \pm 9.5}$ &  101.3 ${\scriptstyle \pm 22.3}$ \\
  &    &      & 0.75 &    \textbf{25.8 ${\scriptstyle \pm 8.8}$} &    29.6 ${\scriptstyle \pm 8.2}$ &    102.0 ${\scriptstyle \pm 6.1}$ &  105.7 ${\scriptstyle \pm 22.8}$ \\
  &    &      & 1.00 &    \textbf{22.4 ${\scriptstyle \pm 7.1}$} &    25.8 ${\scriptstyle \pm 6.4}$ &   107.0 ${\scriptstyle \pm 10.4}$ &  105.3 ${\scriptstyle \pm 12.2}$ \\
  \cline{2-8}
  &    &      & 0.00 &    \textbf{54.0 ${\scriptstyle \pm 9.8}$} &    57.6 ${\scriptstyle \pm 7.9}$ &   234.7 ${\scriptstyle \pm 42.9}$ &                               NA \\
  &    &      & 0.25 &   \textbf{43.8 ${\scriptstyle \pm 10.3}$} &    67.0 ${\scriptstyle \pm 4.0}$ &   269.3 ${\scriptstyle \pm 36.7}$ &                               NA \\
  & 30 & 27\% & 0.50 &    \textbf{39.2 ${\scriptstyle \pm 8.6}$} &   72.4 ${\scriptstyle \pm 15.5}$ &   262.0 ${\scriptstyle \pm 15.7}$ &                               NA \\
  &    &      & 0.75 &    \textbf{35.0 ${\scriptstyle \pm 9.9}$} &   72.2 ${\scriptstyle \pm 10.7}$ &   232.7 ${\scriptstyle \pm 24.0}$ &                               NA \\
  &    &      & 1.00 &    \textbf{29.0 ${\scriptstyle \pm 6.5}$} &    60.3 ${\scriptstyle \pm 5.7}$ &   236.0 ${\scriptstyle \pm 42.8}$ &                               NA \\
  \cline{2-8}
  &    &      & 0.00 &   \textbf{69.0 ${\scriptstyle \pm 11.7}$} &   99.0 ${\scriptstyle \pm 30.7}$ &   460.0 ${\scriptstyle \pm 47.8}$ &                               NA \\
  &    &      & 0.25 &   \textbf{56.8 ${\scriptstyle \pm 15.4}$} &  107.0 ${\scriptstyle \pm 39.2}$ &   438.3 ${\scriptstyle \pm 73.1}$ &                               NA \\
  & 40 & 20\% & 0.50 &   \textbf{50.4 ${\scriptstyle \pm 13.0}$} &  105.6 ${\scriptstyle \pm 41.6}$ &   457.7 ${\scriptstyle \pm 21.2}$ &                               NA \\
  &    &      & 0.75 &   \textbf{41.4 ${\scriptstyle \pm 10.7}$} &   97.8 ${\scriptstyle \pm 33.9}$ &   426.3 ${\scriptstyle \pm 52.6}$ &                               NA \\
  &    &      & 1.00 &   \textbf{34.4 ${\scriptstyle \pm 11.3}$} &   93.5 ${\scriptstyle \pm 16.3}$ &   458.7 ${\scriptstyle \pm 55.3}$ &                               NA \\
\midrule
  &    &      & 0.00 &    \textbf{51.2 ${\scriptstyle \pm 7.5}$} &   56.6 ${\scriptstyle \pm 10.4}$ &   112.0 ${\scriptstyle \pm 12.3}$ &  117.7 ${\scriptstyle \pm 11.9}$ \\
  &    &      & 0.25 &    44.0 ${\scriptstyle \pm 6.5}$ &   \textbf{37.8 ${\scriptstyle \pm 10.2}$} &    92.3 ${\scriptstyle \pm 17.0}$ &  117.3 ${\scriptstyle \pm 11.6}$ \\
6 & 20 & 63\% & 0.50 &    42.2 ${\scriptstyle \pm 8.6}$ &   \textbf{29.0 ${\scriptstyle \pm 10.8}$} &    97.0 ${\scriptstyle \pm 11.5}$ &   112.7 ${\scriptstyle \pm 7.6}$ \\
  &    &      & 0.75 &    38.8 ${\scriptstyle \pm 8.3}$ &   \textbf{25.2 ${\scriptstyle \pm 10.0}$} &    94.0 ${\scriptstyle \pm 11.5}$ &   91.3 ${\scriptstyle \pm 16.6}$ \\
  &    &      & 1.00 &    34.0 ${\scriptstyle \pm 5.1}$ &    \textbf{23.8 ${\scriptstyle \pm 7.7}$} &     93.3 ${\scriptstyle \pm 1.5}$ &    86.0 ${\scriptstyle \pm 7.9}$ \\
  \cline{2-8}
  &    &      & 0.00 &    85.4 ${\scriptstyle \pm 6.5}$ &    \textbf{75.8 ${\scriptstyle \pm 4.4}$} &   256.7 ${\scriptstyle \pm 30.6}$ &                               NA \\
  &    &      & 0.25 &   79.8 ${\scriptstyle \pm 12.5}$ &   \textbf{69.2 ${\scriptstyle \pm 15.4}$} &   260.7 ${\scriptstyle \pm 12.4}$ &                               NA \\
  & 30 & 41\% & 0.50 &  \textbf{69.4 ${\scriptstyle \pm 14.8}$} &   80.6 ${\scriptstyle \pm 17.2}$ &   257.3 ${\scriptstyle \pm 18.6}$ &                               NA \\
  &    &      & 0.75 &   \textbf{67.0 ${\scriptstyle \pm 12.2}$} &   86.2 ${\scriptstyle \pm 23.3}$ &   245.3 ${\scriptstyle \pm 23.7}$ &                               NA \\
  &    &      & 1.00 &   \textbf{57.4 ${\scriptstyle \pm 11.3}$} &    82.0 ${\scriptstyle \pm 5.3}$ &   259.7 ${\scriptstyle \pm 19.5}$ &                               NA \\
  \cline{2-8}
  &    &      & 0.00 &   \textbf{95.4 ${\scriptstyle \pm 27.7}$} &  107.8 ${\scriptstyle \pm 36.3}$ &   460.3 ${\scriptstyle \pm 63.3}$ &                               NA \\
  &    &      & 0.25 &   \textbf{75.6 ${\scriptstyle \pm 23.6}$} &  130.2 ${\scriptstyle \pm 43.3}$ &  401.0 ${\scriptstyle \pm 112.8}$ &                               NA \\
  & 40 & 30\% & 0.50 &   \textbf{63.6 ${\scriptstyle \pm 17.0}$} &  146.0 ${\scriptstyle \pm 51.7}$ &   370.0 ${\scriptstyle \pm 87.1}$ &                               NA \\
  &    &      & 0.75 &   \textbf{57.4 ${\scriptstyle \pm 19.6}$} &  128.3 ${\scriptstyle \pm 16.3}$ &   387.7 ${\scriptstyle \pm 81.1}$ &                               NA \\
  &    &      & 1.00 &   \textbf{47.2 ${\scriptstyle \pm 12.2}$} &  114.5 ${\scriptstyle \pm 31.8}$ &   469.0 ${\scriptstyle \pm 28.2}$ &                               NA \\
\midrule
  &    &      & 0.00 &   \textbf{53.6 ${\scriptstyle \pm 10.2}$} &   82.8 ${\scriptstyle \pm 14.5}$ &   117.7 ${\scriptstyle \pm 11.4}$ &  111.7 ${\scriptstyle \pm 16.5}$ \\
  &    &      & 0.25 &    \textbf{51.0 ${\scriptstyle \pm 8.1}$} &   58.2 ${\scriptstyle \pm 12.6}$ &   108.3 ${\scriptstyle \pm 19.7}$ &   96.7 ${\scriptstyle \pm 15.8}$ \\
8 & 20 & 84\% & 0.50 &    47.0 ${\scriptstyle \pm 5.3}$ &   \textbf{41.4 ${\scriptstyle \pm 13.8}$} &    100.7 ${\scriptstyle \pm 7.6}$ &   91.0 ${\scriptstyle \pm 23.5}$ \\
  &    &      & 0.75 &    40.8 ${\scriptstyle \pm 6.7}$ &    \textbf{26.0 ${\scriptstyle \pm 3.8}$} &     73.7 ${\scriptstyle \pm 8.1}$ &   86.0 ${\scriptstyle \pm 31.2}$ \\
  &    &      & 1.00 &    43.0 ${\scriptstyle \pm 9.0}$ &    \textbf{19.8 ${\scriptstyle \pm 4.9}$} &    81.7 ${\scriptstyle \pm 22.4}$ &   79.7 ${\scriptstyle \pm 10.8}$ \\
  \cline{2-8}
  &    &      & 0.00 &   111.8 ${\scriptstyle \pm 9.8}$ &   \textbf{93.4 ${\scriptstyle \pm 13.3}$} &   255.0 ${\scriptstyle \pm 27.1}$ &                               NA \\
  &    &      & 0.25 &    90.8 ${\scriptstyle \pm 7.5}$ &    \textbf{75.6 ${\scriptstyle \pm 8.2}$} &   242.7 ${\scriptstyle \pm 14.6}$ &                               NA \\
  & 30 & 55\% & 0.50 &   89.6 ${\scriptstyle \pm 13.2}$ &   \textbf{78.8 ${\scriptstyle \pm 12.2}$} &   255.0 ${\scriptstyle \pm 22.1}$ &                               NA \\
  &    &      & 0.75 &    \textbf{77.6 ${\scriptstyle \pm 9.3}$} &    81.2 ${\scriptstyle \pm 9.4}$ &   226.3 ${\scriptstyle \pm 15.1}$ &                               NA \\
  &    &      & 1.00 &    \textbf{71.0 ${\scriptstyle \pm 4.6}$} &   81.6 ${\scriptstyle \pm 12.5}$ &   222.3 ${\scriptstyle \pm 13.0}$ &                               NA \\
  \cline{2-8}
  &    &      & 0.00 &  \textbf{150.4 ${\scriptstyle \pm 16.9}$} &  151.0 ${\scriptstyle \pm 23.1}$ &   450.0 ${\scriptstyle \pm 44.5}$ &                               NA \\
  &    &      & 0.25 &  \textbf{127.0 ${\scriptstyle \pm 12.4}$} &  188.4 ${\scriptstyle \pm 27.4}$ &    452.0 ${\scriptstyle \pm 7.0}$ &                               NA \\
  & 40 & 41\% & 0.50 &  \textbf{113.4 ${\scriptstyle \pm 20.2}$} &  200.0 ${\scriptstyle \pm 33.9}$ &   424.3 ${\scriptstyle \pm 43.0}$ &                               NA \\
  &    &      & 0.75 &  \textbf{104.4 ${\scriptstyle \pm 21.3}$} &  193.0 ${\scriptstyle \pm 14.1}$ &   426.7 ${\scriptstyle \pm 34.6}$ &                               NA \\
  &    &      & 1.00 &   \textbf{92.0 ${\scriptstyle \pm 17.6}$} &   190.0 ${\scriptstyle \pm nan}$ &   422.0 ${\scriptstyle \pm 18.1}$ &                               NA \\
\bottomrule
\end{tabular}
}
    \caption{Detailed results of SHD means and standard deviations from \Cref{app:fig:empirical:interventional_full}. GIES failed to run on $d\ge 30$.}
    \label{app:tab:empirical:interventional}
\end{table}
}
\clearpage

\begin{figure}[p]
    \centering
    \includegraphics[width=\linewidth]{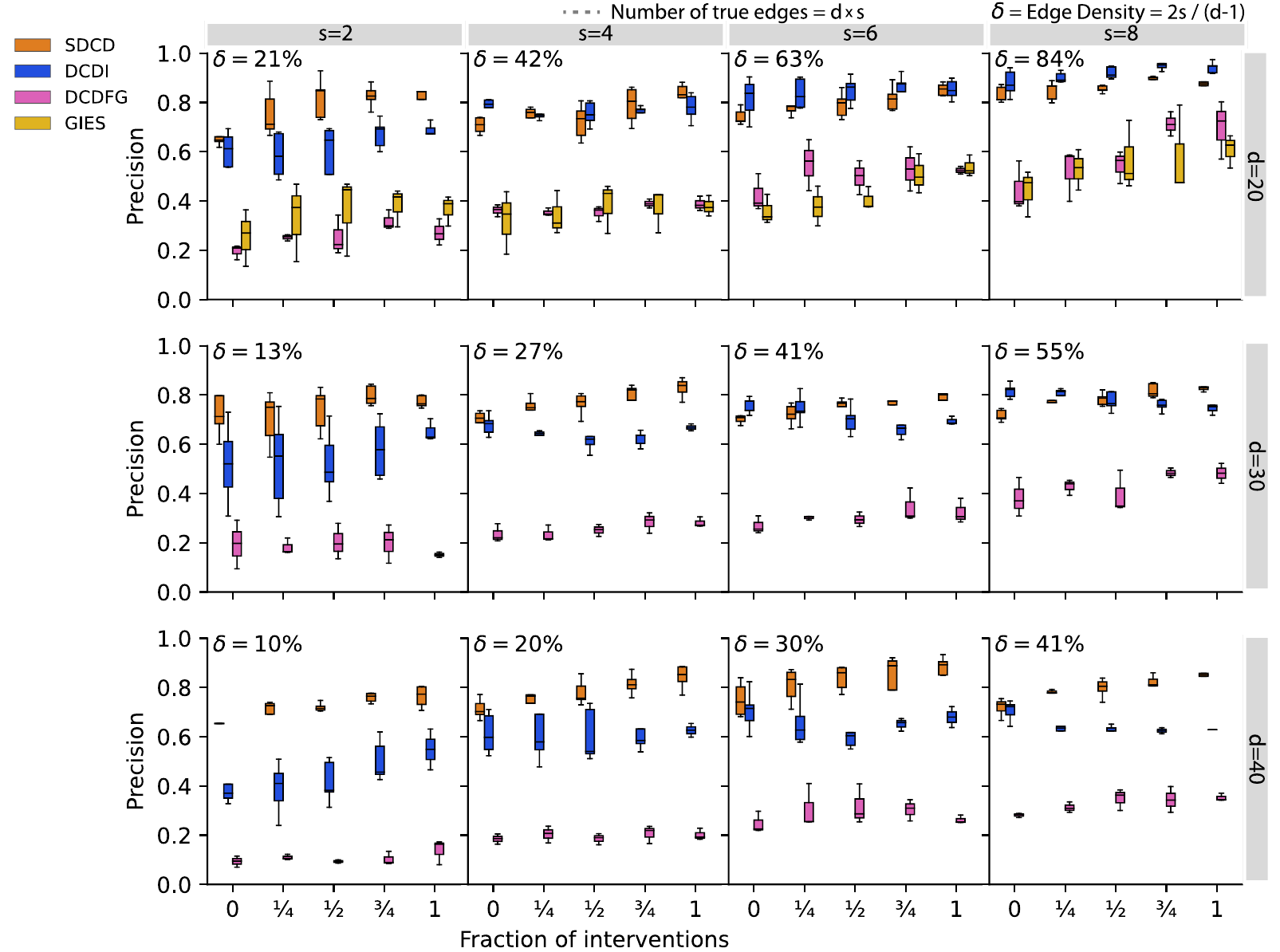}
    \caption{Precision across simulations from \Cref{app:fig:empirical:interventional_full} increasing proportion of variables intervened on, varying the total number of variables $d$ (columns) and average edges per variable $s$ (rows).
    Boxplots over 5 random datasets.}
    \label{app:fig:empirical:interventional_full_precision}
\end{figure}
\clearpage

\begin{figure}[p]
    \centering
    \includegraphics[width=\linewidth]{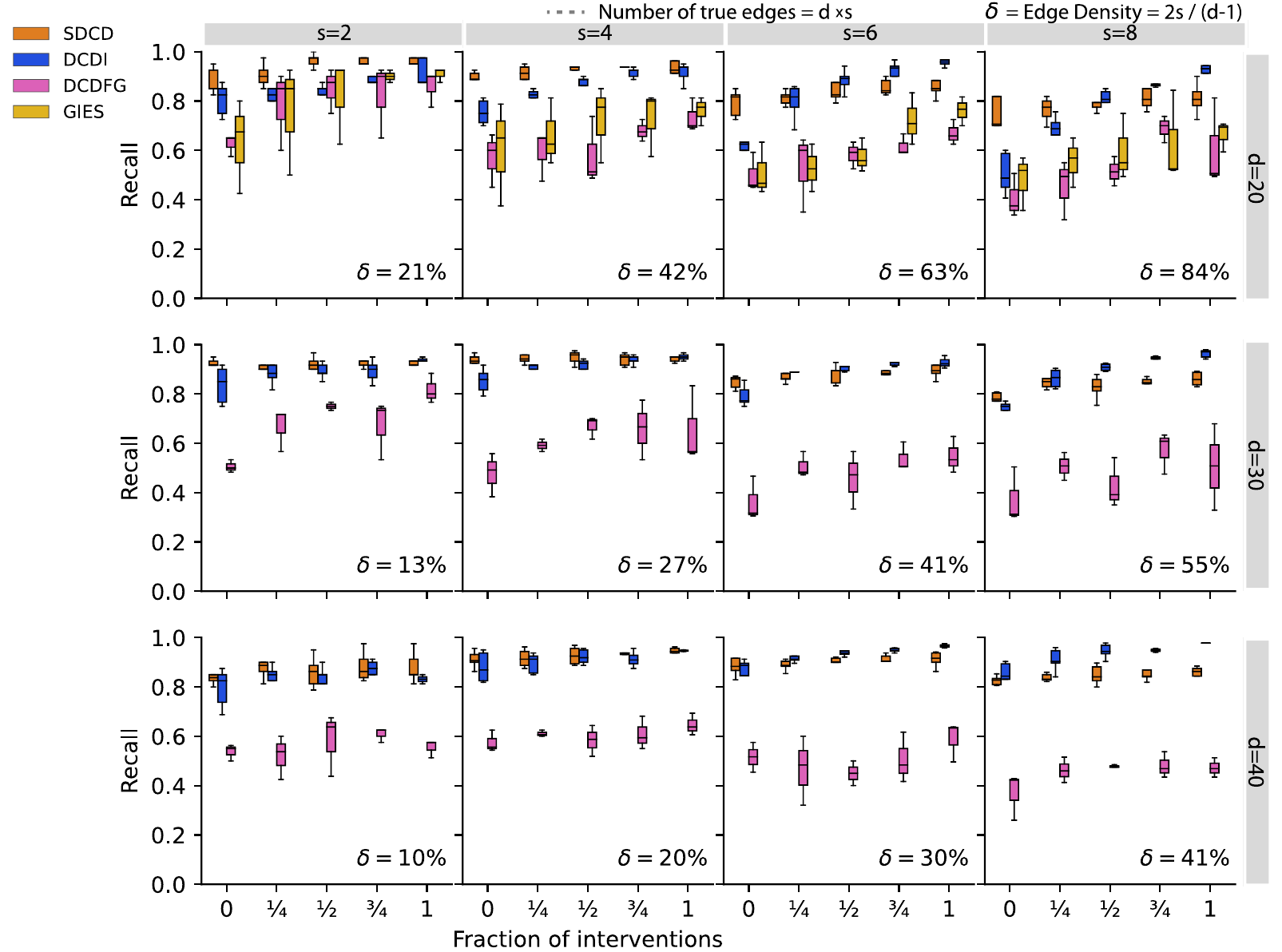}
    \caption{Recall across simulations from \Cref{app:fig:empirical:interventional_full} increasing proportion of variables intervened on, varying the total number of variables $d$ (columns) and average edges per variable $s$ (rows).
    Boxplots over 5 random datasets.}
    \label{app:fig:empirical:interventional_full_recall}
\end{figure}
\clearpage

\begin{table}[p]
    \centering
    \resizebox{0.5\columnwidth}{!}{%
\begin{tabular}{lrrrr}
\toprule
Name &    d=10 &    d=20 &     d=30 &     d=40  \\
\midrule
SDCD             &  \textbf{14.7} &  40.3 &   \textbf{54.3} &   69.0  \\
SDCD-warm        &  \textbf{14.7} &  40.7 &   55.0 &   \textbf{68.7}  \\
SDCD-warm-nomask &  19.3 &  69.7 &  156.0 &  272.7  \\
SDCD-no-s1       &  19.3 &  68.3 &  155.3 &  272.3  \\
SDCD-no-s1-2     &  16.3 &  56.7 &   95.0 &  135.0  \\
\midrule
DCDI             &  24.0 &  \textbf{35.7} &   56.7 &   87.0  \\
\bottomrule
\end{tabular}
}
    \caption{Ablation study for SDCD stage 1. We observe that the described version of SDCD performs the best out of all variations. SDCD-warm performs competitively but generally provides little benefit. SDCD-warm-nomask performs much worse than SDCD, demonstrating that enforcing the mask during stage 2 is important. We report mean SHD over three random seeds of observational data (no interventions) with a fixed number of edges per variable, $s=4$, for a range of numbers of variables, $d$. \textbf{SDCD-warm} refers to starting stage 2 of SDCD, where the input layer is ported over from stage 1 instead of re-learned. \textbf{SDCD-warm-nomask} performs the same warmstart as SDCD-warm but does not enforce the mask in stage 2. \textbf{SDCD-no-s1} only performs stage 2. \textbf{SDCD-no-s1-2} only does stage 2, but sets $(\alpha_2, \beta_2)$ to the default values from stage 1 $(\alpha_1, \beta_1)$. We report these values alongside DCDI. The lowest SHD values are bolded for each value of $d$.}
    \label{app:tab:empirical:ablation}
\end{table}

\begin{table}[p]
    \centering
    \resizebox{0.55\columnwidth}{!}{%

\begin{tabular}{lrrrrr}
\toprule
Name & $d=10$ & $d=20$ & $d=30$ & $d=40$ & $d=50$ \\
\midrule
SDCD & 13.33 & \textbf{33.47} & \textbf{54.07} & \textbf{70.80} & \textbf{76.60} \\
SDCD-exp & 11.60 & 44.33 & 69.07 & 85.07 & 89.93 \\
SDCD-log & \textbf{11.20} & 52.00 & 87.47 & 117.87 & 116.00 \\
\bottomrule
\end{tabular}
}
\caption{
Ablation study for SDCD stage 2 (choice of the constraint). We observe that SDCD performs the best out of the three variations. Additionally, the variations using the PST constraints do not crash for any of the runs, even for those with $d=50$. We attribute this improved stability (as compared to DCDI and DAGMA) to stage 1 since there are fewer non-zero parameters contributing to the value of the constraint. We report mean SHD over five random seeds of observational data (no interventions) with a fixed number of edges per variable, $s=4$, for a range of numbers of variables, $d$. \textbf{SDCD-exp} is SDCD except using the $h_\text{exp}$ constraint in place of the $h_\rho$, and \textbf{SDCD-log} uses the $h_\text{log}$ constraint.}
    \label{app:tab:empirical:ablation2}
\end{table}